\documentclass[11pt]{article}
\usepackage{tikz}
\usepackage{pgfplots}
\usepackage{amsthm}
\usepackage{amsmath}
\usepackage{amssymb }
\usepackage{array}
\usepackage{calc}
\usepackage{xcolor}
\usepackage{cite}
\usepackage{authblk}
\usetikzlibrary{decorations.pathreplacing}

\newlength{\boxwidth}
\DeclareMathOperator*{\E}{\mathbb{E}}
\DeclareMathOperator*{\R}{\mathbb{R}}

\newcommand{\OPT} {
  \mathrm{OPT}}

\newcommand{\DT} {
  \mathrm{DT}}

\newcommand{\AD} {
  \mathrm{AD}}

\newcommand{\NAD} {
  \mathrm{NAD}}

\newcommand{\MQ} {
  \mathrm{MQ}}

\newcommand{\Rand} {
  \mathrm{R}}

\newcommand{\MP} {
  \mathrm{MP}}

\newcommand{\ignore}[1]{}

\mathchardef\mhyphen="2D

\let\emptyset\varnothing

\newtheorem{theorem}{Theorem}
\newtheorem{corollary}[theorem]{Corollary}
\newtheorem{lemma}[theorem]{Lemma}

\newtheorem{defn}{Definition}

\begin{document}

\title{Adaptive Exact Learning of Decision Trees \\ from Membership Queries}
\date{}
\author[1]{ Nader H. Bshouty}
\author [2]{Catherine A. Haddad-Zaknoon}
\affil[1,2]{ Department of Computer Science - Technion ,   Haifa, 32000, Israel}
\affil[1]{ bshouty@cs.technion.ac.il}
\affil[2]{catherine@campus.technion.ac.il}


\maketitle
\begin{abstract}
In this paper we study the adaptive learnability of decision trees of depth at most $d$ from membership queries. This has many applications in automated scientific discovery such as drugs development and software update problem. Feldman solves the problem in a randomized polynomial time algorithm that asks $\tilde O(2^{2d})\log n$ queries and Kushilevitz-Mansour in a deterministic polynomial time algorithm that asks $ 2^{18d+o(d)}\log n$ queries. We improve the query complexity of both algorithms. We give a randomized polynomial time algorithm that asks $\tilde O(2^{2d}) + 2^{d}\log n$ queries and a deterministic polynomial time algorithm that asks $2^{5.83d}+2^{2d+o(d)}\log n$ queries.

\end{abstract}


\section{Introduction}
Learning from membership queries, \cite{Ang88}, has fueled the interest of many papers due to its diverse applications in various areas including group testing \cite{DH00}, pooling design for DNA sequencing \cite{DH06},  blood testing \cite{Dor43}, functional genomics (molecular biology) \cite{KWJRB04},  chemical reactions \cite{AC08, ABM14}, multi-access channel communications \cite{DH00, BG07}, whole-genome shotgun sequencing \cite{ABKRS02}, DNA physical mapping \cite{GK98}, game theory \cite{Pel02}, program synthesis \cite{BDVY17}, channel leak testing, codes, VLSI testing and AIDS screening and many others as in \cite{BC16} and the references therein. Generally, in domains where queries can be answered by a non-human annotator, e.g. from experiments such as laboratory tests, learning from membership query may be a promising direction for automated scientific discovery.

In this paper, we consider the problem of learning Boolean decision trees from membership queries provided that the tree is of depth at most $d$. Learning decision trees from membership queries is equivalent to solving problems in widespread areas including drug development, software update problem, bioinformatics, chemoinformatics  and many others.
\subsection{Applications}
In the sequel we bring two examples where learning decision tree in the exact learning model might be applied to improve scientific and technological research.

\textbf{Drug development.} The process of drugs discovery starts by searching for lead compounds in a huge library of chemical molecules \cite{GRR08, DAG06, GPP05}. A compound is called \emph{active} with respect to some (biological) problem if it exhibits a very high activity in a  biological test or \emph{assay}. For example, checking if some molecule binds with a specific protein or cancerous cell. The essential goal is  to find, at a very early stage, some lead compounds  that exhibit a very high activity in the assay. One of the famous methods to cope with this  problem is building \emph{Quantitive Structure-Activity Relationship} model or \emph{QSAR}. The \emph{QSAR} methodology focuses on finding a  model, which correlates between \emph{activity} and the compound structure. A molecule is described using a predefined set of features or \emph{descriptors}  selected according to  chemical considerations. Decision trees are one of the standard ways used in drug research to model such relations. A membership query can be simulated by building a compound according to the specific descriptors defined by the algorithm's query and screening it against the biological assay. Compounds that exhibit high activity against the assay are related to as positive queries,  while inactive compounds are related as negative ones.

\textbf{The software update problem. } Complicated software systems are made up of many interacting packages that must be deployed and coexist in the same context. One of the tough challenges system administrators face is to know which software upgrades  are safe to apply to their particular environment without breaking the running and working system, thus causing downtime. The software upgrade problem is to determine which updates to perform without breaking the system \cite{BGS18}. Breaking the system might be as a result of \emph{conflicting} packages, \emph{defective} packages or \emph{dependent} packages.  Let $P$ denote the set of \emph{packages}.  An installation is a set of packages $I\subseteq P$.  An installation can be successful depending on whether all dependencies, conflicts and defect constraints are satisfied. Let $p,q\in P$. A \emph{dependency} $p\rightarrow q$ is a constraint that means any successful installation that contains the package $p$ must contain $q$. A \emph{conflict} $\{p,q\}$ implies that any installation that includes both $p$ and $q$ together will fail. Finally, a package $p$ is a \emph{defect} if it cannot be part of any successful installation. Consider that the dependencies, conflicts and defects are not known to the user, the purpose is to learn a model that for any installation can determine whether the installation will succeed or not. Let $P=\{p_1,\cdots, p_n\}$, $R$ be the set of the dependencies, $C$ the set of conflicting packages and $D$ the set of defectives. Let $x_1, \cdots, x_n$ be the variables indicating whether the package participates in the installation or not. This is equivalent to learning decision tree of depth $d$ from membership queries where a membership query is simulated by actually making the installation and getting a status feedback about it and $d =O(|C| + |R| + |D|)$. For the cases where $d<n$, learning decision trees of depth at most $d$ might be an efficient solution.

\subsection{The Learning Model}
In the {\it exact learning from membership queries} model~\cite{Ang88}, we have a \emph{teacher}  and a \emph{learner}. The teacher holds a hidden \emph{target} function or \emph{concept} $f:\{0,1\}^n \rightarrow \{0,1\}$ from a fixed class of Boolean functions $C$ that is known to the learner. The target of the learner is to discover the function $f$. Given an access to the teacher, the learner can ask the teacher {\it membership queries} by sending it elements $a\in \{0,1\}^n$. As a response, the teacher returns an answer $f(a) \in \{0,1\}$.  The main target of the learner is to learn the exact target function~$f$, using a minimum number of membership queries and a small time complexity. We regard the learner as a \emph{learning algorithm}  and the teacher as an \emph{oracle} $\MQ_f$ that answers membership queries, i.e., $\MQ_f(a)=f(a)$.

Let $C$  and $H$ be two classes of representations of Boolean functions. Let $C$ be a subset of $H$ in a sense that each function in $C$ has a representation in $H$.
We say that a learning algorithm $A$ \emph{exactly learns} $C$ {\it from} $H$ in time $t(A)$ {\it with} (or, and {\it asks}) $q(A)$ membership queries if for every target function $f\in C$, the algorithm $A$ runs in time at most $t(A)$, asks at most $q(A)$ membership queries and outputs a function $h\in H$ that is logically equivalent to $f$ ($f = h$). If such an algorithm exists, then we say that the class $C$ is \emph{learnable from} $H$ in time $t(A)$ {\it with} $q(A)$ membership queries (or just, queries).

The learning algorithm can be \emph{deterministic} or \emph{randomized}, \emph{adaptive} (AD) or \emph{non-adaptive} (NAD). In the adaptive learning algorithm, the queries might rely on the results of previous ones. On the other hand, in the non-adaptive algorithm, the queries are independent and must not rely on previous results. Therefore, all the queries can be asked in a single parallel step. We say that an adaptive algorithm is an \emph{$r$-round adaptive} ($r$-RAD) if it runs in $r$ stages where each stage is non-adaptive. That is, the queries may depend on the answers of the queries in previous stages, but are independent of the answers of the current stage queries. 

A randomized algorithm is called a  \emph{Monte Carlo} (MC) algorithm if it uses coin tosses and its running time is finite, while its output might be incorrect with a probability of at most $\delta$, that is specified as an input to the algorithm. All of the randomized algorithms regarded herein are Monte Carlo algorithms.

Let $\DT_d$ denote the class of decision trees of depth at most $d$.  We say that the class $\DT_d$ is \emph{polynomial time learnable} if there is an algorithm that learns $\DT_d$ in $poly(2^d,n)$ time and asks $poly(2^d, \log n)$ membership queries. We cannot expect better time or query complexity since just asking one membership query takes $O(n)$ time. As we will see below, the lower bound for the number of queries is $\Omega(2^d\log n)$.

Along the discussion, we will use the notation $\OPT(\DT_d)$ to indicate the optimal number of membership queries needed to learn the class $\DT_d$ with unlimited computational power. More precisely, we use the notation $\OPT_{\AD}(\DT_d)$ ($\OPT_{\NAD}(\DT_d)$, $\OPT_{\Rand,\AD}(\DT_d)$, $\OPT_{\Rand,\NAD}(\DT_d)$) to indicate the optimal number of membership queries needed by a deterministic adaptive (resp., deterministic non-adaptive, randomized adaptive with with failure probability of  $\delta=1/2$, randomized non-adaptive with failure probability of $\delta=1/2$) algorithm to exact-learn the class $\DT_d$ from membership queries.

\subsection{Known Results}
In the following, we give a brief summary of some of the known results of the learnability of the class $\DT_d$.
In \cite{Bsh17,BC16}, it is shown that\footnote{$\tilde O(T)$ is $O(poly(d)\cdot T)$. We will use this notation throughout this paper. The exact upper bound is $O(d2^{2d}\log n)$.}:
\begin{theorem}
$$\tilde O(2^{2d})\log n \geq \OPT_{\NAD}(\DT_d) \geq \OPT_{\AD}(\DT_d)   \geq \Omega(2^d \log n) $$ and
$$\tilde O(2^{2d})\log n \geq \OPT_{\Rand,\NAD}(\DT_d)\geq \OPT_{\Rand,\AD}(\DT_d)   \geq \Omega(2^d \log n).$$

\end{theorem}
Achieving an upper bound better than $\tilde O(2^{2d})\log n$ is an open problem. Table \ref{tab:results} summarizes some known results on learning decision trees. Most of the algorithms for learning decision trees relate to the target function using the Discrete Fourier Transform representation and try to learn its non-zero Fourier coefficients. Vitaly Feldman ~\cite{VF07}  introduced the first randomized {\it non-adaptive} algorithm that finds the heavy Fourier coefficients. Applying his algorithm to the class $\DT_d$ gives a non-adaptive learning algorithm with query complexity $\tilde O(2^{2d})\log^2n$. See \cite{Bsh17} for the algorithm. Applying the reduction from \cite{BC16} to the algorithm of Feldman gives a non-adaptive learning algorithm with query complexity     $\tilde O(2^{2d}) \log n$. We want to draw the reader's attention that any result for the non-adaptive case implies the same result for the adaptive case. This is result (1) in Table \ref{tab:results}.
Kushilevitz and Mansour~\cite{KM93}, and Jackson ~\cite{Jck97}, gave a randomized adaptive algorithm for learning decision trees known as KM algorithm. Kushilevitz and Mansour then use the technique of $\epsilon$-biased distribution to derandomize the algorithm. A rough estimation of the query complexity of their algorithm for the class of $\DT_d$ was first done in~\cite{Bsh17}. In Appendix \ref{KM}, we show that, with the best derandomization results exist today~\cite{T17}, the query complexity of their deterministic algorithm is $ 2^{18d + o(d)}\log n$. This is result (2) in Table \ref{tab:results}.

\begin{table}\centering
 \begin{tabular}{||c | c|c |c||}
 \hline
Deterministic/&Algorithm& Query \\ [0.5ex]
Randomized&&Complexity\\
 \hline\hline
Randomized $^{(1)}$ & \cite{VF07,BC16 } & $\tilde O(2^{2d})\log n$ \rule{0pt}{2.4ex}\\
\hline
Randomized & Ours& $\tilde O(2^{2d})+2^d\log n$\rule{0pt}{2.4ex} \\
\hline\hline
 Deterministic$^{(2)}$ & \cite{KM93, Jck97,BHL95 }&$2^{18d+o(d)}\log n$ \rule{0pt}{2.4ex}\\
 \hline
Deterministic&Ours  &$2^{5.83d}+2^{2d+o(d)}\log n$ \\
\hline
\end{tabular}
\caption{Query complexity for polynomial time adaptive algorithms}
\label{tab:results}\end{table}

\subsection{Our Results}
In this paper we give two new algorithms. The first algorithm is a randomized adaptive algorithm (two rounds) that asks $\tilde O(2^{2d})+2^d\log n$ membership queries. In particular, this shows that $\OPT_{\Rand,\AD}(\DT_d)\le \tilde{O}(2^{2d})+2^d\log n$. The second algorithm is a deterministic adaptive algorithm that asks $2^{5.83d}+2^{2d+o(d)}\log n$. See Table~\ref{tab:results}.

For the randomized adaptive algorithm, we give a new reduction that changes a randomized $r$-round algorithm $A$ that learns a class $C$ with $Q(n)$ membership queries to a randomized $(r+1)$-round algorithm that learns $C$ with $Q(8V^2)+v\log n$ queries, where $V$ is an upper bound on the number of relevant variables of the functions in $C$ and $v$ is the number of relevant variables of the target function. The class $C$ must be closed under variable projection. We apply this reduction to the randomized nonadaptive (one round) algorithm of Feldman for $\DT_d$. In Feldman's algorithm we have $Q(n)=\tilde O(2^{2d})\log n$. Since, for $\DT_d$, $V,v\le 2^d$, we get a two-round algorithm that asks $\tilde O(2^{2d})+2^d\log n$ queries.

In our reduction, we randomly uniformly project the variables $(x_1,\ldots,x_n)$ of the target function $f$ into $m=8V^2$ new variables $(y_1,\ldots,y_{m})$. Since the number of the relevant variables of the target function is at most $V$, with high probability all the variables in the target function are projected to distinct variables. Let $g(y)$ be the projected function. We run the algorithm $A$ to learn $g(y)$. Obviously, queries to $g$ can be simulated by queries to $f$. This takes $Q(8V^2)$ queries. Algorithm $A$ returns a function $h(y)$ equivalent to $g$. Then, we show, without asking any more membership queries, how to find the relevant variables of $g(y)$ and how to find witnesses for each relevant variable. That is, for each relevant variable $y_i$ in $g$, we find two assignments $a^{(i)}$ and $b^{(i)}$ that differ only in $y_i$ such that $g(a^{(i)})\not=g(b^{(i)})$. Moreover, we show that for each relevant variable $y_i$ in $g$, one can find the corresponding relevant variable in $f$ by running the folklore nonadaptive algorithm of group testing for finding one defective item combined with the witnesses. Each relevant variable can be found with $\log n$ queries. This requires $v\log n$ queries in the second round.

For the deterministic adaptive algorithm, we first find the relevant variables. Then, we learn the decision tree over the relevant variables using the multivariate polynomial representation,~\cite{BM95}, with the derandomization studied in~\cite{B14}. To find the relevant variables, we construct a combinatorial block design, that we will call an $(n,d)$-{\it universal disjoint set}, in polynomial time. A set $S\subseteq \{0,1,z\}^n$ is called an $(n,d)$-universal disjoint set if for every $1\le i_1<\cdots<i_d\le n$, every $\xi_1,\ldots,\xi_d\in\{0,1\}^n$ and every $1\le j\le d$ there is an assignment $a\in S$ such that $a_{i_k}=\xi_k$ for all $k\not=j$ and $a_{i_j}=z$. We regard $z$ as a variable. The algorithm finds all the relevant variables in the target function as follows. To find the relevant variables, the algorithm starts with constructing an $(n,2d+1)$-universal disjoint set $S$. For any assignment $a\in S$, let $a[x]$ be the assignment $a$ where each entry $i$ that is equal to $z$ is replaced with $x_i$. We show that for every decision tree $f$ and every relevant variable $x_i$ in $f$ there is an assignment $a\in S$ such that $f(a[x])$ is equal to either $x_i$ or $\bar x_i$. Therefore, when $f(a[1^n]|_{x_i\gets 1})\not=f(a[0^n]|_{x_i\gets 0})$, we can run the algorithm that learns one literal to learn a relevant variable in $f$. To learn $f(a[x])$, we need $O(\log n)$ queries. This gives all the relevant variables of~$f$ in $|S|+O(v\log n)$ queries where $v$ is the number of relevant variables in $f$.

In addition, in this paper we give a polynomial time algorithm that constructs an $(n,2d+1)$-universal disjoint set $S$ of size $|S|=2^{2d+o(d)}\log n$. This size is optimal up to $2^{o(d)}$. Hence, the total number of queries for the first stage is $2^{2d+o(d)}\log n$.

In the second stage, the algorithm learns the target function over the relevant variables that are found in the first stage. We start by showing that $\DT_d$ can be represented as $\MP_{d,3^d}$; multivariate polynomial over the field $F_2$ with monomials of size (number of variables in the monomial) at most $d$ and at most $3^d$ monomials. To learn $\MP_{d,s}$, $s=3^d$, over $N\le 2^d$ variables, we first show how to learn one monomial. In order to (deterministically) learn one monomial, we use the deterministic zero test procedure  Zero-Test of $\MP_{d,s}$ from \cite{B14} that asks $Z=2^{2.66d}s\log^2N$ queries. The learning algorithm of $\MP_{d,s}$ from~\cite{BM95} with this procedure gives a deterministic learning algorithm that asks $O(ZNs)$ queries. This solves the problem with query complexity $\tilde O(2^{6.83d})$ (for learning $\DT_{d}$ with $N\le 2^d$ variables). In this paper, we give an algorithm that asks $O(Zs\cdot \log N)$ queries that gives query complexity $\tilde O(2^{5.83d})$.

To find one monomial, our algorithm uses $(N,d)$-{\it sparse all one set} $D\subseteq \{0,1\}^N$. An $(N,d)$-sparse all one set is a set $D\subseteq \{0,1\}^N$ where for every $1\le i_1<\cdots<i_d\le n$ there is an assignment $a\in D$ such that $a_{i_k}=1$ for all $k=1,\ldots,d$ and the weight of each $a\in D$ is at most $N(1-1/2d)$. We show how to construct such set in polynomial time.
Thereafter, we show that for every $f\in\DT_d$, $f\not=0$, there is an assignment $a\in D$ such that $f(a*x)\not=0$ where $a*x=(a_1x_1,\ldots,a_Nx_N)$. To find such an assignment, we use the procedure Zero-Test to zero test each $f(a*x)$, for all $a\in D$. Notice that the weight of $a$ is at most $(1-1/2d)N$ and therefore, the function $f(a*x)$ contains at most $(1-1/2d)N$ variables. Thus, if we recursively use this procedure, after at most $2d\ln N$ iterations, we get a monomial $M$ in $f$. Therefore, to find one monomial we need at most $2^{2.66d}s 2d\ln N$ queries.

Assuming that we have found $t$ monomials $M_1,\ldots,M_t$ of $f$, we can learn a new one by learning a monomial of $f'=f+M_1+M_2+\cdots+M_t$. The function $f'$ is the function $f$ without the monomials $M_1,\ldots,M_t$. This is because $M+M=0$ in the binary field. Obviously, we can simulate queries for $f'$ with queries for $f$. Since the number of monomials in $f$ is at most $s=3^d$, the number of queries in the second stage is at most $2^{2.66d}s^2 2d\ln N= \tilde O(2^{5.83d})$.
\section{Preliminaries}
Let $n$ be an integer. We denote $[n] = \{1,\cdots,n\}$. Consider the set of \emph{assignments}, also called \emph{membership queries} or \emph{queries}, $\{0,1\}^n$. A function $f$ is called a \emph{Boolean function} on $n$ variables $\{x_1, \cdots, x_n\}$ if $f: \{0,1\}^n\rightarrow \{0,1\}$. We denote by $1^n$ ($0^n$) the all one (zero) assignment.  For an assignment $a$, the $i$-th entry in $a$ is denoted by $a_i$. For two assignments $a,b\in \{0,1\}^n$, the assignment $a+ b$ denotes the bitwise xor of the assignments $a$ and $b$. For an assignment $a\in\{0,1\}^n$, $i \in [n]$ and $\xi \in \{0,1\}$, we denote by $a|_{x_i\leftarrow \xi}$ the assignment $(a_1, \cdots, a_{i-1}, \xi, a_{i+1}, \cdots, a_n)$. We say that the variable $x_i$ is \emph{relevant} in $f$ if there is an assignment $a = (a_1, \cdots, a_n)$   such that $f(a|_{x_i\leftarrow 0}) \neq f(a|_{x_i\leftarrow 1})$.  We call such assignment a \emph{witness} assignment for $x_i$. A Boolean decision tree is defined as follows:
\begin{defn}
Let $\DT$ denote the set of all Boolean \emph{decision tree} formulae. Then, 1) $f\equiv 0 \in \DT$ and $f \equiv 1\in \DT$ and 2) if $f_0 , f_1\in \DT$, then for all $x_i$, $f:= \overline{x}_if_0 + x_if_1 \in \DT$.
\label{DTDefn}
\end{defn}
An immediate observation from the construction of $f$, as defined by step~2 in Definition \ref{DTDefn}, is that the terms of $f$ are disjoint i.e. there is at most one non-zero term for the same assignment. Consequently, the ``$+$'' operation can be replaced with the logic ``$\vee$''. Hence, we can relate to the function $f$ as a disjunction of pairwise disjoint terms.

Every decision tree $f$ can be \emph{modelled} as a binary tree denoted by $T(f)$ or shortly $T$ if $f$ is known from the context. Let $\cal{V}$ denote the vertices set of $T$.  If $f \equiv 0$ or $f \equiv 1$, then $T$ is a node labelled with $0$ or $1$ respectively. Otherwise, by Definition \ref{DTDefn}, there are some $x_i$ and $f_0 , f_1\in \DT$ such that $f = \overline{x}_if_0 + x_if_1$. Then, $T$  has a root node $r\in \cal{V}$ labelled with $x_i$ and two outgoing edges. The left edge is labelled by zero and points to the root node of $T(f_0)$. The right edge is labelled by $1$ and points to the root node of $T(f_1)$. The \emph{size} of a decision tree $T$ is the number of its leaves. For any tree $T$ we define $depth(T)$ to be the number of the edges of the longest path from the root to a leaf in $T$.  Let $T_f$ denote all the decision tree models of the Boolean function $f$.  Let $d$ be the depth of the decision tree model $T(f)$. We say that $T(f)$ is a \emph{minimal} model for $f$ if $d = \mathop{\min}_{T\in T_f} depth(T)$. The \emph{depth} of a decision tree $f$ is the depth of its minimal model. We denote by $\DT_d$ the class of all decision trees of depth at most $d$.

Given an assignment $a \in \{0,1\}^n$, the decision tree model $T(f)$ defines a computation. The computation traverses a path from the root to a leaf and assigns values to the variables indicated by the inner nodes in the following way. The computation starts from the root node of the tree $T(f)$. When the computation arrives to an inner node labeled by $x_i$, it assigns the value of the input $a_i$ to $x_i$. If $a_i = 1$, then the computation proceeds to the right child of the current node. Otherwise, it proceeds to the left child. The computation terminates when a leaf $u$ is reached. The value of the computation is the value indicated by the leaf $u$. 
\begin{lemma}\label{DToDT} Let $f_1,f_2\in \DT_d$, then for any boolean function $g(y_1,y_2)$ we have $g(f_1,f_2)\in \DT_{2d}$.
\end{lemma}
\begin{proof} We construct the following tree denoted by $[f_1,f_2]_g$. We replace each leaf $v$ of $f_1$ by the tree $f_2$. Denote by $w_{u,v}$ the leaf $u$ of the tree $f_2$ that replaces $v$ in $f_1$.  The label of $w_{u,v}$ is $g(\ell(v),\ell(u))$, where $\ell(v)$ is the label of $v$ in $f_1$ and $\ell(u)$ is the label of $u$ in $f_2$. Therefore, any assignment $a$ that reaches the leaf $v$ in $f_1$ and $u$ in $f_2$ satisfies $g(f_1,f_2)(a)=g(f_1(a),f_2(a))=g(\ell(v),\ell(u))$. On the other hand, this assignment reaches the leaf $w_{u,v}$ in $[f_1,f_2]_g$ which has the label $g(\ell(v),\ell(u))$. The result follows since $[f_1,f_2]_g\in \DT_{2d}$.
\end{proof}

\subsection{Multivariate Polynomial Representation of Decision Trees }
A \emph{monotone term} or \emph{monomial} is a product of variables. The {\it size} of a monomial is the number of the variables in it. A \emph{multivariate polynomial} is a linear combination (over the binary field $F_2$) of distinct monomials. The size of a multivariate polynomial is the number of monomials in it. It is well known that multivariate polynomials have a unique representation so the size is well defined. Let $\MP_{d,s}$ be the set of Boolean multivariate polynomials over the binary field $F_2$ with at most $s$ monomials each of size at most $d$. We say that a set $Z\subseteq \{0,1\}^n$ is a {\it zero test} for $\MP_{d,s}$ if  for every $f \in \MP_{d,s}$ and $f \neq 0$, there is an assignment $a\in Z$ such that $f(a)\not=0$. In \cite{B14}(Lemma 72), it is proved: 
\begin{lemma}\label{ZT1} There is a zero test set $Z\subseteq \{0,1\}^n$  for $\MP_{d,s}$ of size $2^{2.66d}s\log^2n$ that can be constructed in polynomial time $poly(2^d, n)$.
\end{lemma}
\begin{lemma}\label{DTMP} Let $f\in \DT_d$. Then, $f$ is equivalent to some $h\in \MP_{d,3^d}$.
\end{lemma}
\begin{proof} The proof is by induction on $d$. For $d=0$, the decision tree is either $0$ or $1$. The size is at most $1 \leq 3^0$. For $d=1$, the decision tree is either  $x_i$ or $x_i+1$. The size is at most $2\le 3^1$.  Suppose the hypothesis is true for $d-1$. Let $f$ be a decision tree of depth $d$ and let $x_i$ be the label of its root. Then,  $f=x_if_1+(x_i+1)f_2$ where $f_1$ and $f_2$ are trees of depth $d-1$. By the induction hypothesis, $f_1,f_2$ are equivalent to $h_1,h_2\in \MP_{d-1,3^{d-1}}$. Since $f=x_if_1+x_if_2+f_2=x_ih_1+x_ih_2+h_2\in \MP_{d,3^{d}}$,  the result follows.
\end{proof}

\subsection{Block Designs}
In this section we define two block designs that will be used in the deterministic algorithm for learning decision trees.

A set $S\subseteq \{0,1,z\}^n$ is called an $(n,d)$-{\it universal disjoint set} if for every $1\le i_1<\cdots<i_d\le n$, every $\xi_1,\ldots,\xi_d\in\{0,1\}$ and every $1\le j\le d$, there is an assignment $a\in S$ such that $a_{i_k}=\xi_k$ for all $k\not=j$ and $a_{i_j}=z$. We regard $z$ as a variable.

Let $H$ be a family of functions $h:[n]\to [q]$. For $d\le q$, we say that $H$ is $(n,q,d)$-{\it perfect hash family} if for every subset $S\subseteq [n]$ of size $|S|=d$ there is a {\it hash function} $h\in H$ such that $|h(S)|=d$.

Bshouty ~\cite{B15} shows that:
\begin{lemma}\label{BD02} Let $q$ be a power of prime. If $q>2d^2$, then there is an $(n,q,d)$-perfect hash family of size
$$O\left(\frac{d^2\log n}{\log(q/(2d^2))}\right),$$
that can be constructed in linear time.
\end{lemma}
The following lemma is proved in~\cite{NSS95}:
\begin{lemma}\label{BD03}
An $(8d^2,d)$-universal set of size $2^{d+O(\log^2d)}$ can be constructed in $poly(2^d, n)$ time.
\end{lemma}
We now prove:
\begin{lemma} \label{UDS} There is an $(n,d)$-{\it universal disjoint set} $S\subseteq \{0,1,z\}^n$ of size $2^{d+O(\log^2 d)}$ $\log n$ that can be constructed in polynomial time.
\end{lemma}
\begin{proof} Consider the $(8d^2,d)$-universal set $U$ of size $2^{d+O(\log^2d)}$ from Lemma~\ref{BD03}.  Define a set $W\subseteq \{0,1,z\}^t$, for $t=8d^2$ as follows. For each $a\in U$, add $t$ vectors $a^{(1)},\ldots,a^{(t)}$ to $W$ where $a^{(i)}=a|_{x_i\gets z}=(a^{(i)}_1,\ldots,a^{(i)}_{i-1},z,a^{(i)}_{i+1},\ldots,a^{(i)}_t)$. Obviously, $W$ is $(8d^2,d)$-universal disjoint set of size $\ell=2^{d+O(\log^2 d)}$. Let $w^{(1)},\ldots,w^{(\ell)}$ be the vectors in $W$. Let $4d^2<q=2^r\le 8d^2$, and consider an $(n,q,d)$-perfect hash family $H$. By Lemma~\ref{BD02}, $H$ is of size $O(d^2\log n)$, and can be constructed in linear time. We use $W$ and $H$ to construct an $(n,d)$-universal disjoint set $S$ as follows. For every $h\in H$ and every $w\in W$, we add the assignment $[w,h]=(w_{h(1)},\ldots,w_{h(n)})$ to $S$.

First, the size of $S$ is $2^{d+O(\log^2 d)}\log n$. Moreover, consider $1\le i_1<\cdots<i_d\le n$ and any $\xi_1,\ldots,\xi_d\in\{0,1\}$. Since $H$ is an $(n,q,d)$-perfect hash family, there is a hash function $h\in H$ such that $r_1=h(i_1),\ldots,r_d=h(i_d)\in [q]$ are distinct. Since $W$ is an $(8d^2,d)$-universal disjoint set and $q\le 8d^2$ for every $1\le j\le d$, there is an assignment $w\in S$ such that $w_{r_k}=\xi_k$ for all $k\not=j$ and $w_{r_j}=z$. Therefore, $[w,h]_{i_k}=w_{h(i_k)}=w_{r_k}=\xi_k$ and $[w,h]_{i_j}=z$. Thus, $S$ is an $(n,d)$-universal disjoint set.
\end{proof}
An {\it $(n,d)$-sparse all one set} $D\subseteq \{0,1\}^n$ is a set such that for every $a\in D$, $wt(a):=|\{i|a_i=1\}|\le n(1-1/(2d))$, and for every $1\le i_1<\cdots<i_d\le n$ there is an assignment $a\in D$ such that $a_{i_k}=1$ for all $k\in [d]$. 
\begin{lemma} \label{NDSprs}There is an $(n,d)$-sparse all one set $D\subseteq \{0,1\}^n$ of size $d+1$ that can be constructed in polynomial time.
\end{lemma}
\begin{proof} Consider the $d+1$ vectors $a^{(j)}\in \{0,1\}^n$, $j=0,1,\ldots,d$ where $a^{(j)}_i=0$ if $i\mod (d+1)=j$ and $a^{(j)}_i=1$ otherwise. Consider any $1\le i_1<i_2<\cdots<i_d\le n$. 
Let $j_t = i_t$   $mod (d+1)$ for $t \in [d]$. Therefore,  $j_1,\cdots, j_d$ take at most $d$ distinct values from $S=\{0,1,\cdots, d\}$. As a result, there is $k\in S$ such that $j_\ell\not=k$ for all $\ell=1,\ldots,d$. Then, obviously by the construction of $D$, $a^{(k)}_{i_1}=\cdots=a^{(k)}_{i_d}=1$. Moreover, for each $j$, $wt(a^{(j)})\le n-\lceil n/(d+1)\rceil\le n(1-1/(2d)).$
\end{proof}

\section{Randomized Adaptive Learning}
In this section we give a randomized two-round algorithm that learns $\DT_d$ in polynomial time with $\tilde O(2^{2d})+2^d\log n$ queries. First, we give a reduction that changes a randomized $r$-round learning algorithm to a randomized $(r+1)$-round learning algorithm while reducing the number of queries for classes with small number of relevant variables. We use this reduction for $\DT_d$ with Feldman's algorithm to get the above result.

An $m$-{\it variable projection} is a function $P:\{x_1,\ldots,x_n\}\to \{y_1,\ldots,y_m\}$ where $\{y_1,\ldots,y_m\}$ are new distinct variables. For $x=(x_1,\ldots,x_n)$ we write $P(x)=(P(x_1),\ldots,$ $P(x_n))$, where for an assignment $a = (a_1,\cdots, a_n)$, $P(x_i)(a) = a_j$ if $P(x_i) =y_j $. Let $C$ be a class of Boolean functions. We say that $C$ is closed under {\it variable projection} if for every $f\in C$, every integer $m$ and every $m$-variable projection $P$ we have $f(P(x))\in C$. We say that $C$ is $t(n)$-{\it verifiable} (with success probability at least $1-\delta$) if for every $f\in C$ with $n$ variables one can find, in time $t(n)$ (with probability at least $1-\delta$),  the relevant variables of $f$, and can derive a witness for each relevant variable $x_i$, i.e. an assignment $a^{(i)}\in\{0,1\}^n$ such that $f(a^{(i)}|_{x_i\gets 1})\not=f(a^{(i)}|_{x_i\gets 0})$. We denote by $V(f)$ the number of relevant variables of $f$. Let $V(C)=\max_{f\in C} V(f)$. The reduction is described in Figure~\ref{FT}. Now we can prove:
\begin{theorem}\label{main1} Let $C$ be a class of functions that is closed under variable projection. $m=8V(C)^2$. Let $H$ be a $t(n)$-verifiable class with success probability at least $15/16$. Let ${\cal A}_n$ be a randomized $r$-round algorithm that learns the class $C$ from $H$ in time $T(n)$ with $Q(n)$ membership queries and success probability at least $15/16$.
Then, there is a randomized $(r+1)$-round algorithm that learns $C$ from $H$ in time $O(T(m)+nQ(m)+t(m)+V(f)n\log n)$ with $Q(m)+V(f)(\log n+1)$ queries and success probability at least $3/4$, where $f$ is the target function.
\end{theorem}
\begin{proof}
Let $P$ be a random uniform $m$-variable projection (step 1 in Figure~\ref{FT}). The algorithm runs ${\cal A}_{m}$ to learn $g(y)=f(P(x))$ (step 2). Since $C$ is closed under variable projection, $g(y)=f(P(x))\in C$. Therefore, with probability at least $15/16$, ${\cal A}_m$ returns a function $h(y)$ equivalent to $g(y)$.
In Lemma~\ref{probfail}, we show that, with probability at least $15/16$, the relevant variables of $f$ are mapped by $P$ to different $y_i$. Therefore, with probability at least $7/8$, $h(y)$ is equivalent to $g(y)$ and the relevant variables of $f$ are mapped by $P$ to different $y_i$. We now proceed assuming such events are true.

Since $g(y)=h(y)\in H$ and $H$ is $t(n)$-verifiable with success probability at least $15/16$, one can find the relevant variables of $g$ and witnesses for each relevant variable in time $t(m)$ and success probability at least $15/16$. That is, for each relevant variable $y_i$ in $g$, the procedure Find-Witness in step (3), with probability at least $15/16$, finds an assignment $a^{(i)}\in \{0,1\}^m$ and $\xi_i\in \{0,1\}$ such that $g(a^{(i)}|_{y_i\gets \xi_i})=0$ and $g(a^{(i)}|_{y_i\gets \bar \xi_i})=1$. The procedure Find-Witness returns a set $Y$ of all the relevant variables of $g(y)$ and a set $A$ containing the corresponding witnesses of the relevant variables in $g$.

Notice that, for all $y_\ell\in Y$, $g(a^{(\ell)}|_{y_\ell\gets \xi_\ell+z})=z$, $z\in\{0,1\}$. Consider some $y_\ell\in Y$ and let $X_\ell:=P^{-1}(y_\ell)=\{x_j\ |\ P(x_j)=y_\ell\}$. That is, all the variables $x_j$ that are mapped to $y_\ell$ by $P$. Notice that, if $X$ is the set of relevant variables of $f$, then for each $y_\ell\in Y$, the set $X\cap P^{-1}(y_\ell)$ contains exactly one variable that maps to $y_\ell$.
To find this variable, suppose $P(x)=(y_{j_1},\ldots,y_{j_n})$. Then $g(a^{(\ell)})=f(a^{(\ell)}_{j_1},\ldots,a^{(\ell)}_{j_n})$. Since $f$ has exactly one relevant variable in $X_\ell$ and $g(a^{(\ell)}|_{y_\ell\gets \xi_\ell+z})=z$, we have $f(\beta^{(\ell)})=g(a^{(\ell)}|_{y_\ell\gets \xi_\ell+z})=z$ where $\beta^{(\ell)}=(\beta_1^{(\ell)},\ldots,\beta_n^{(\ell)})$, $\beta_i^{(\ell)}=\xi_\ell+z$ if $x_i\in X_\ell$ and $\beta_i^{(\ell)}=a^{(\ell)}_{j_i}$ if $x_i\notin X_\ell$.
Therefore, $f(b^{(\ell)})=x_t$ where $b^{(\ell)}_i=\xi_\ell+x_i$ if $x_i\in X_\ell$ and $b^{(\ell)}_i=a^{(\ell)}_{j_i}$ if $x_i\notin X_\ell$. Therefore, by learning $f(b^{(\ell)})$ we find the relevant variable in $X_\ell$. These are steps 4.1-4.3. In Lemma~\ref{learnvar}, we show how to learn one variable non-adaptively with at most $\log n+1$ queries. This is the procedure Learn-Variable in step 4.4. It returns $k_\ell$ where $x_{k_\ell}$ is the variable that corresponds to $y_\ell$. Finally, we replace each relevant variable $y_\ell\in Y$ in $g$ with $x_{k_\ell}$ and get the target function (Step 5).

The success probability of the algorithm is at least $13/16>3/4$. The time complexity in step 2 is $O(T(m)+nQ(m))$, in step 3 is $t(m)$, and in step~4.4, by Lemma~\ref{learnvar}, is $O(V(f)n\log n)$. This gives $O(T(m)+nQ(m)+t(m)+V(f)n\log n)$ time complexity. The query complexity is $Q(m)$ in step 2 and $V(f)\lceil \log n\rceil$ in step 4.4. The number of rounds is $r$ rounds for running ${\cal A}_m$ and one round for finding
the relevant variables in step~4.
\end{proof}
In Lemma \ref{relwit}, we show that if the class $C$ is closed under variable projection and is learnable from $H$ in time $T(n)$ and $Q(n)$ queries, then it is verifiable in time $t(n)=O((T(n)+Q(n)\tau(n))n\log n)$ where $\tau(n)$ is the time to compute a function in $H$. In particular, the result of Theorem~\ref{main1} is true in time $O((T(m)+Q(m)\tau(m))m\log m+nQ(m)+V(f)n\log n).$
In particular, we have:
\begin{corollary} The class $\DT_d$ is $2$-round learnable with a randomized algorithm that runs in polynomial time and asks $\tilde O(2^{2d})+2^d\log n$ queries.
\end{corollary}
\begin{proof} We use Feldman's non-adaptive (one round)  randomized algorithm that asks $Q(n)=O(d^32^{2d}\log^2 n)$ queries and runs in $T(n)=poly(2^d,n)$ time. In  Lemma~\ref{witnes1} in Appendix \ref{KM}, we show that the output of Feldman's algorithm is verifiable in time $O(2^{2d} n)$. For $\DT_d$, we have $V(C)=2^d$. By Theorem~\ref{main1}, we get a two round randomized algorithm that runs in time $poly(2^d,n)$ and asks $O(d^52^{2d}+2^d\log n)$ queries.
\end{proof}

\begin{figure}[h!]
  \begin{center}
  \fbox{\fbox{\begin{minipage}{28em}
  \begin{tabbing}
  xxxx\=xxxx\=xxxx\=xxxx\= \kill
  {\bf Learn$(C,n)$}\\
  ${\cal A}_n$ is an algorithm that learns $C$ in time $T(n)$ with $Q(n)$ queries \\ and let $m=8V(C)^2$.\\
  1) Let $P$ be random uniform $m$-variable projection.\\
  2) Run ${\cal A}_m$ $\ \ \ \backslash*$ Learning $g(y)=f(P(x))$ $*\backslash$\\
  \> 2.1) For each membership query $a\in \{0,1\}^m$ \\
  \>\>\> ask MQ$_f$($P(x_1)(a),\ldots,P(x_n)(a)$) $\ \ \ \backslash*$ $y_i(a)=a_i *\backslash$ \\
  \> 2.2) Let $h$ be the output.\\
  3) $(Y,A)\gets$Find-Witness($h(y)$)\\
  4) For each relevant variable $y_\ell\in Y$ and a witness $a^{(\ell)}\in A$ of $y_\ell$\\
  \> 4.1) Define $\xi_\ell=1$ if $h(a^{(\ell)}|_{y_\ell\gets 1})=0$ and $\xi_\ell=0$ otherwise.\\
  \> 4.2) Define $X_\ell=P^{-1}(y_\ell)$.\\
  \> 4.3) Define $b^{(\ell)}_i=\xi_\ell+x_i$ if $x_i\in X_\ell$ and $b^{(\ell)}_i=a_j^{(\ell)}$ if $x_i\notin X_\ell$\\
  \> 4.4) ${k_\ell}\gets$Learn-Variable($f(b^{(\ell)})$)\ \ \ $\backslash*$ Using MQ$_f *\backslash$\\
  5) For each $y_\ell\in Y$, replace $y_\ell$ in $h$ with $x_{k_\ell}$.\\
  6) Output($h$).
  \end{tabbing}
  \end{minipage}}}
  \end{center}
	\caption{A reduction from an $r$-round algorithm to an $(r+1)$-round algorithm.}
	\label{FT}
	\end{figure}
\subsection{Proof of the Lemmas}
In this subsection we prove the lemmas used in the previous section.
\begin{lemma}\label{relwit} Let $C$ be a class of functions that is closed under variable projection. Let ${\cal A}_n$ be a randomized learning algorithm that learns $C$ from $H$ in time $T(n)$ with $Q(n)$ queries and success probability at least $15/16$. Let $g(y_1,\ldots,y_m)\in C$ be a computable function in time $\tau(m)$. Then, the relevant variables of $g$ and a witness $a^{(i)}$ for each relevant variable $y_i$ in $g$ can be found in time $O((T(m)+Q(m)\tau(m))m\log m)$ and success probability at least $15/16$.
\end{lemma}

\begin{proof}  We first give an algorithm that for any two functions $g_1,g_2\in C$ in $m$ variables, with success probability at least $1-1/(16m)$, verifies whether they are equivalent, in time $O((T(m)+Q(m)\tau(m))\log m)$. In case they are not equivalent, the algorithm returns an assignment $a$ such that $g_1(a)\not=g_2(a)$.
We run the learning algorithm ${\cal A}_m$. For each query $a$, we check whether $g_1(a)=g_2(a)$. If not, then we are done. Otherwise, stop and return $g_1(a)(=g_2(a))$. If for all the membership queries both functions are equal and the algorithm returns $h$, then $\Pr[g_1\not =g_2]\le \Pr[g_1\not=h]+\Pr[g_2\not=h]\le 1/8$. Running the algorithm $O(\log m)$ times gives failure probability $1/(16m)$.
For each $2\le i\le m$, define $g^{(i)}(y)=g(y_{y_i\gets y_1})$ and $g^{(1)}(y)=g(y_{y_1\gets y_2})$. Since $C$ is closed under variable projection, we have $g^{(i)}\in C$. Notice that $y_i$ is a relevant variable if and only if $g(y)\not=g^{(i)}(y)$. The above algorithm can check whether $g=g^{(i)}$ for all $i$ with success probability at least $15/16$. In case, for some $i$, we get an assignment $a$ such that $g^{(i)}(a)\not=g(a)$ then (assume $i>1$),
$$g^{(i)}(a)=g(a_1,\cdots,a_{i-1},a_1,a_{i+1},\cdots,a_n)\not= g(a).$$ This can happen only when $a_1\not=a_i$ and then the assignment $a$ is a witness for $y_i$.
\end{proof}

\begin{lemma}\label{probfail} Let $1\le i_1\le \cdots\le i_k\le n$. For a random $ck^2$-variable projection $P$ with
probability at least $1-1/(2c)$, $P(x_{i_1}),\ldots,P(x_{i_k})$ are distinct.
\end{lemma}
\begin{proof} The probability that $P(x_{i_j})$ is equal to $P(x_{i_k})$ for some $j\not=k$ is $1/(ck^2)$.
Therefore, by the union bound, the probability that $P(x_{i_1}),\ldots,P(x_{i_k})$ are distinct is at least $1-{k\choose 2}/(ck^2)\ge 1-1/(2c)$.
\end{proof}
In the following, we give a non-adaptive algorithm that learns one variable (the procedure Learn-Variable). This result is a folklore result in group testing.
\begin{lemma}\label{learnvar} There is a non-adaptive learning algorithm that asks $\lceil \log n\rceil$ queries and learns the class $C=\{x_1,\ldots,x_n\}$.
\end{lemma}
\begin{proof} Denote by $b_i(j)$ the $i$-th bit of the binary representation of $j$. Define the queries $a^{(i)}=(b_i(0),\ldots,b_i(n-1))$, $i=1,\ldots,\lceil\log n\rceil$. If the target function is $x_\ell$, then the vector of answers to the queries is the binary representation of $\ell-1$.
\end{proof}

\section{Deterministic Adaptive Learning}\label{ADP}
In the following, we give a deterministic algorithm for learning $\DT_d$. In subsection~\ref{FRV1}, we show that there is a polynomial time deterministic algorithm that finds all the relevant variables in $2^{2d+O(\log^2 d)}\log n$ queries. In subsection~\ref{LMP1}, we show that the class $\MP_{d,s}$, defined over $N$ variables, can be learned in a deterministic polynomial time with $2^{2.66d}s^2\log^3N$ queries. 
\begin{theorem} The class $\DT_d$ is deterministic polynomial time learnable from $\MP_{d,3^d}$ with $2^{5.83d} +2^{2d+O(\log^2d)}\log n$ queries.
\end{theorem}
\begin{proof} By Lemma \ref{qcfr}, we can find the relevant variables in $2^{2d+O(\log^2 d)}\log n$ queries. The number of relevant variables $N$ is bounded by $N\le 2^d$. Since, by Lemma~\ref{DTMP}, each decision tree of depth at most $d$ has a representation in $\MP_{d,3^d}$, we can learn it as a multivariate polynomial over the $N$ variables. By Theorem \ref{mpsd}, this takes $2^{2.66d}s^2\log^3N=d^32^{2.66d}3^{2d}\le 2^{5.83d}$ queries.
\end{proof}
\subsection{Finding the Relevant Variables}\label{FRV1}
In this section we give an algorithm that asks $2^{2d+O(\log^2d)}\log n$ queries and finds all the relevant variables of $f\in \DT_d$. We first give a folklore procedure Binary-Learn-Variable that takes two assignments $a,b\in \{0,1\}^n$ such that $f(a)\not=f(b)$ and finds a relevant variable $x_j$ of $f$ such that $a_j+b_j=1$ in $\lceil \log n\rceil$ queries. The procedure defines an assignment $c$ that results from flipping $\lfloor wt(a+b)/2\rfloor$ entries in $a$ that differ from $b$. Then, it finds $f(c)$ with a membership query. We either have $f(a)\not=f(c)$ or $f(b)\not=f(c)$ because, otherwise, $f(a)=f(b)$ and then we get a contradiction. In the first case, $a$ differs from $c$ in $\lfloor wt(a+b)/2\rfloor$ entries, while in the second case, $c$ differs from $b$ in $\lceil wt(a+b)/2\rceil$ entries. Then, we recursively do the above until, at the $t$th stage, both assignments differ in $s=\lceil wt(a+b)/2^t\rceil=1$ entries. Therefore, the variable that corresponds to this entry is relevant in $f$. The number of iterations until $s=1$ is at most $\lceil \log n\rceil$.

We remind the reader that for $a\in \{0,1,z\}^n$, $a[x]$ is the assignment $a$ where each entry $i$ that is equal to $z$ is replaced with $x_i$. For example, for $a=(1,0,z,1,z)$, $a[x]=(1,0,x_3,1,x_5)$.
\begin{lemma}\label{findX} Let $S\subseteq \{0,1,z\}^n$ be an $(n,2d+1)$-universal disjoint set. If $f\in \DT_d$ depends on $x_i$, then there is an assignment $a\in S$ such that $f(a[x])\in \{x_i,\bar x_i\}$.
\end{lemma}
\begin{proof} If $x_i$ is relevant in $f$, then $f(x|_{x_i\gets 0})\not=f(x|_{x_i\gets 1})$. Therefore, $f'(x)=f(x|_{x_i\gets 0})+f(x|_{x_i\gets 1})\not=0$. Since $f(x|_{x_i\gets 0}),f(x|_{x_i\gets 1})\in\DT_d$, by Lemma~\ref{DToDT} we have $f'(x)\in \DT_{2d}$. Let $x_{i_1}\stackrel{\xi_1}{\to}x_{i_2}\stackrel{\xi_2}{\to}\cdots {\to}x_{i_{d'}}\stackrel{\xi_{d'}}\to 1$, $d'\le 2d$ be a path in the decision tree representation of $f'$. Such path exists since $f'\in \DT_{2d}$ and $f'\not=0$. Obviously,  $i\not\in \{i_1,\ldots,i_{d'}\}$ since $x_i$ is not relevant in $f'$. By the definition of $S$, there is an assignment $a\in S$ such that $a_{i_{1}}=\xi_1,\ldots,a_{i_{d'}}=\xi_{d'}$ and $a_i=z$. Hence, $f'(a)=1$ and $f'(a[x])=1$.
Therefore, either $f(a[x]|_{x_i\gets 0})=1$ and $f(a[x]|_{x_i\gets 1})=0$, or $f(a[x]|_{x_i\gets 0})=0$ and $f(a[x]|_{x_i\gets 1})=1$. As a result, we get $f(a[x])=x_if(a[x]|_{x_i\gets 1})+\bar x_i f(a[x]|_{x_i\gets 0})\in \{x_i,\bar x_i\}$.
\end{proof}

\begin{figure}[h!]
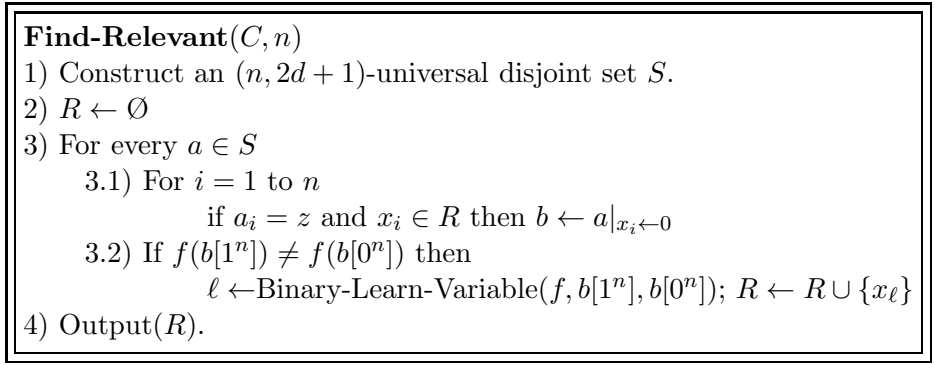

  \begin{center}
  \fbox{\fbox{\begin{minipage}{28em}
  \begin{tabbing}
  xxxx\=xxxx\=xxxx\=xxxx\= \kill
  {\bf Find-Relevant$(C,n)$}\\
  1) Construct an $(n,2d+1)$-universal disjoint set $S$.\\
  2) $R\gets \O$\\
  3) For every $a\in S$ \\
  \> 3.1) For $i=1$ to $n$ \\
  \>\>\> if $a_i=z$ and $x_i\in R$ then $b\gets a|_{x_i\gets 0}$\\
  \> 3.2) If $f(b[1^n])\not= f(b[0^n])$ then \\ \>\>\>$\ell\gets$Binary-Learn-Variable($f,b[1^n],b[0^n]$);
   $R\gets R\cup \{x_\ell\}$ \\
  4) Output($R$).
  \end{tabbing}
  \end{minipage}}}
  \end{center}
	\caption{An algorithm that returns all the relevant variables.}
	\label{NRV}
	\end{figure}
Figure~\ref{NRV} describes the algorithm. In step 1, the algorithm constructs an $(n,2d+1)$-universal disjoint set $S$.
In step 2, it defines an empty set $R$ that eventually, will contain all the relevant variables. Then, in step 3, for every $a\in S$, it defines a new assignment $b$ that is derived from $a$ by assigning to zero all $a_i=z$ in $a$ that indicate a known relevant variable $x_i\in R$. Thereafter, if $f(b[1^n])\not=f(b[0^n])$, the algorithm runs Binary-Learn-Variable$(f,b[1^n],b[0^n])$.  The procedure Binary-Learn-Variable satisfies the following:
\begin{lemma}\label{FR01}
Let $a\in\{0,1,z\}^n$. If $f(a[1^n])\not=f(a[0^n])$ then, the procedure call Binary-Learn-Variable $(f,a[1^n],a[0^n])$
finds, in at most $\lceil\log n\rceil$ queries, a relevant variable $x_j$ that
satisfies:\\
1) $a[1^n]_j+a[0^n]_j=1$, i.e., $a_j=z$.\\
2) $x_j$ is a relevant variable in $f(a[x])$ and therefore in $f$.
\end{lemma}
The algorithm adds, in step 3.2, the learned relevant variable to the set $R$. We prove:
\begin{lemma}\label{NQ} The number of times that Binary-Learn-Variable is executed is at most $V(f)$, where $V(f)$ the is the number of the relevant variables of $f$.
\end{lemma}
\begin{proof} Notice that after step 3.1, for every $x_i\in R$ we have $a_i=0$ and hence $b_i=0$ too.
Therefore, for every $x_i\in R$, we have $b[1^n]_i=b[0^n]_i$. By Lemma~\ref{FR01},
the relevant variable that Binary-Learn-Variable finds is not in $R$. Also, by Lemma~\ref{FR01}, all the variables in $R$ are relevant in $f$. Thus, Binary-Learn-Variable cannot be executed more than the number of relevant variables of $f$ i.e. $V(f)$.
\end{proof} We now show:
\begin{lemma} The procedure Find-Relevant returns all the relevant variables in $f$.
\end{lemma}
\begin{proof} Let $x_j$ be a relevant variable in $f$. By Lemma~\ref{findX},
there is an assignment $a\in S$ such that $f(a[x])\in \{\bar{x_j},x_j\}$.
If $x_j \in R$ when the algorithm reaches $a$, then we are done. Otherwise, if $x_j \notin R$, then, the algorithm defines a new assignment $b$ from $a$ such that for each $x_i\in R$ and $a_i=z$ we have $b_i=0$,  otherwise, $b_i=a_i$. In particular, $b_j=z$. Therefore, $f(b[x])\in \{\bar{x_j},x_j\}$. This is true because, if $c\in \{0,1,z\}^n$ and $f(c[x])=x_j$ (or $\bar x_j$), then for any $k\not=j$ where $c_k=z$, $f(c|_{x_k\gets 0}[x])=f(c[x]|_{x_k\gets 0})=x_j$  (or $\bar x_j$). Therefore, $f(b[1^n])\not=f(b[0^n])$ and the algorithm runs Binary-Learn-Variable$(f,b[1^n],b[0^n])$. By Lemma~\ref{FR01}, the procedure Binary-Learn-Variable finds a relevant variable in $f(b[x])\in \{\bar{x_j},x_j\}$ and therefore finds $x_j$.
\end{proof}
\begin{lemma}
\label{qcfr}
The query complexity of Find-Relevant is $2^{2d+O(\log^2 d)}\log n$.
\end{lemma}
\begin{proof} 
Steps 3.1-3.2 are executed $|S|$ times and therefore, by Lemma~\ref{UDS}, the number of queries asked in the If condition is $2|S|=2^{2d+O(\log^2 d)}\log n$. Moreover, by Lemma~\ref{NQ}, Binary-Learn-Variable is executed $V(f)$ times and, by Lemma~\ref{FR01}, in each execution it asks $\lceil\log n\rceil$ queries. Therefore, the number of queries asked in Binary-Learn-Variable is $V(f)\lceil\log n\rceil=O(2^d\log n)$.
\end{proof}
\subsection{Learning MP}\label{LMP1}
In this section we give an algorithm that learns $\MP_{d,s}$ with $2^{2.66d}s^2\log^3 N$ queries, where $N$ is the number of variables. We first show how to learn one monomial.
\begin{figure}[h!]
  \begin{center}
  \fbox{\fbox{\begin{minipage}{28em}
  \begin{tabbing}
  xxxx\=xxxx\=xxxx\=xxxx\= \kill
  {\bf Learn-Monomial$(f,N)$}\\
  1) If Zero-Test$(f)=${\it True} then\\
  \>\> Output(``The function is identically zero'') and Halt\\
  2) $b\gets 1^N$; \\
  3) While {\it True} do \\
  \> 3.1) Construct a $(wt(b),\min(wt(b)-1,d))$-sparse all one set $D$.\\
  \> 3.2) For all $a\in D$ \\
  \>\> \> If Zero-Test$(f((a\Delta b)*x))=${\it false} then Goto 3.4\\
  \> 3.3) Goto 4\\
  \> 3.4) $b\gets a\Delta b$\\
  4) Output$\left(\bigwedge_{b_i=1}x_i\right)$.
  \end{tabbing}
  \end{minipage}}}
  \end{center}
	\caption{An algorithm that learns one monomial.}
	\label{LOM}
	\end{figure}
The algorithm is in Figure~\ref{LOM}. In the algorithm, the procedure Zero-Test answers ``True'' if $f$ is identically zero and ``False'' otherwise. It uses Lemma~\ref{ZT1}. Recall that, $1^N$ is the all one assignment of dimension $N$, $wt(b)$ is the Hamming weight of $b$ and $a*x=(a_1x_1,\ldots,a_Nx_N)$. For two assignments $b\in \{0,1\}^N$ and $a\in \{0,1\}^{wt(b)}$,  the assignment $a\Delta b$ is defined
as follows. Let $1\le i_1<i_2<\cdots<i_{wt(b)}$ be all the indices $j$ such that $b_j=1$. Then $(a\Delta b)_{i_\ell}=a_{\ell}$ for all $\ell=1,2,\ldots,wt(b)$ and $(a\Delta b)_j=0$ for all $j\not\in\{i_1,\ldots,i_{wt(b)}\}$. For example, if $b=(10011101)$ and $a=(10110)$, then $a\Delta b=(a_100a_2a_3a_40a_5)=(10001100)$. 
\begin{lemma}\label{YYY} If $f(b*x)\not=0$ then,
\begin{enumerate}
\item If $f((a\Delta b)*x)\not=0$ for some $a\in D$, then $wt(a\Delta b)  \le wt(b)(1-1/(2 d))$.
\item
If $f((a\Delta b)*x)=0$ for all $a\in D$, then $M=\wedge_{b_i=1}x_i$ is a monomial in $f$.
\end{enumerate}
\end{lemma}
\begin{proof} Since $D$ is $(wt(b),\min(wt(b)-1,d))$-sparse all one set and $a\in D$, $wt(a\Delta b) = wt(a) \le wt(b)(1-1/(2d))$.
This implies {\it 1.}
To prove {\it 2.}, suppose $f((a\Delta b)*x)=0$ for all $a\in D$.
Let $M=x_{i_1}\cdots x_{i_{d'}}$, $d'\le d$ be any minimal size monomial in $f(b*x)$. In particular, $b_{i_1}=\cdots=b_{i_{d'}}=1$. From the definition of $\Delta$, we have $(a\Delta b)_{i_\ell}=a_{j_\ell}$, $\ell=1,\ldots,d'$, for some $1\le j_1<j_2<\cdots < j_{d'}\le wt(b)$. We now prove that $wt(b)=d'$. Suppose, for the sake of contradiction, that $wt(b)\ge d'+1$. Then, $d''=\min(wt(b)-1,d)\ge d'$. Since $D$ is $(wt(b),d'')$-sparse all one set, there is an assignment $a\in D$ such that $a_{j_\ell}=1$ for all $\ell=1,2,\ldots,d'$. Thus, $(a\Delta b)_{i_\ell}=1$ for all $\ell=1,2,\ldots,d'$. As a result, $f((a\Delta b)*x)$ will contain the monomial $M$ and therefore is not zero. This is a contradiction. Therefore, $wt(b)=d'$. Since $wt(b)=d'$ and $b_{i_1}=\cdots=b_{i_{d'}}=1$, we have $\wedge_{b_i=1}x_i=x_{i_1}\cdots x_{i_{d'}}=M$.
\end{proof}
\begin{lemma}\label{kuku} The number of iterations in step 3 is at most $O(d\log N)$.
\end{lemma}
\begin{proof} At the beginning $b=1^N$ and, by Lemma~\ref{YYY}, at each iteration the weight of $b$ is reduced by a factor of $1-1/(2d)\le e^{-1/(2d)}$. This implies the result.
\end{proof}
Lemma~\ref{kuku} implies that:
\begin{lemma}\label{lll} Algorithm Learn-Monomial deterministically learns one monomial in $f$ with $\tilde O(2^{2.66d})s\log^3N$ queries.
\end{lemma}
\begin{proof} By Lemma~\ref{ZT1}, each zero test takes $2^{2.66d}s\log^2N$ queries. By Lemma~\ref{kuku}, the number of iterations in step 3 is at most $O(d\log N)$.
\end{proof}
We now prove:
\begin{theorem} \label{mpsd}There is a deterministic learning algorithm that learns the class $\MP_{s,d}$ in $\tilde O(2^{2.66d})s^2 \log^3N$ queries.
\end{theorem}
\begin{proof} Assume we have found $t$ monomials $M_1,\ldots,M_t$ of $f$ we can learn another new monomial by learning a monomial of $f'=f+M_1+M_2+\cdots+M_t$. The function $f'$ is just the function $f$ without the monomials in $M_1,\ldots,M_t$. This is because $M+M=0$ in the binary field. Using Lemma~\ref{lll}, the result follows.
\end{proof}


\vskip 0.2in
\bibliography{Ref}

\begin{thebibliography}{10}

\bibitem{Ang88}
D.~Angluin, ``Queries and concept learning,'' {\em Machine Learning}, vol.~2,
  no.~4, pp.~319--342, 1987.

\bibitem{DH00}
D.-Z. Du and F.~K. Hwang, {\em Combinatorial Group Testing and Its
  Applications}.
\newblock World Scientfic Publishing, 1993.

\bibitem{DH06}
D.-Z. Du and F.~K. Hwang, {\em Pooling Designs And Nonadaptive Group Testing:
  Important Tools For Dna Sequencing}.
\newblock World Scientfic Publishing, 2006.

\bibitem{Dor43}
R.~Dorfman, ``The detection of defective members of large populations,'' {\em
  The Annals of Mathematical Statistics}, vol.~14, no.~4, pp.~436--440, 1943.

\bibitem{KWJRB04}
R.~D. King, K.~E. Whelan, F.~M. Jones, P.~G.~K. Reiser, C.~H. Bryant, S.~H.
  Muggleton, D.~B. Kell, and S.~G. Oliver, ``Functional genomic hypothesis
  generation and experimentation by a robot scientist,'' {\em Nature},
  vol.~427, 2004.

\bibitem{AC08}
D.~Angluin and J.~Chen, ``Learning a hidden graph using $o(\log n)$ queries per
  edge,'' {\em J. Comput. Syst. Sci.}, vol.~74, no.~4, pp.~546--556, 2008.

\bibitem{ABM14}
H.~Abasi, N.~H. Bshouty, and H.~Mazzawi, ``On exact learning monotone {DNF}
  from membership queries,'' {\em CoRR}, vol.~abs/1405.0792, 2014.

\bibitem{BG07}
E.~Biglieri and L.~Györfi, {\em Multiple Access Channels: Theory and
  Practice.}
\newblock IOS Press, 2007.

\bibitem{ABKRS02}
N.~Alon, R.~Beigel, S.~Kasif, S.~Rudich, and B.~Sudakov, ``Learning a hidden
  matching,'' in {\em 43rd Symposium on Foundations of Computer Science {(FOCS}
  2002), 16-19 November 2002, Vancouver, BC, Canada, Proceedings}, p.~197,
  2002.

\bibitem{GK98}
V.~Grebinski and G.~Kucherov, ``Reconstructing a hamiltonian cycle by querying
  the graph: Application to {DNA} physical mapping,'' {\em Discrete Applied
  Mathematics}, vol.~88, no.~1-3, pp.~147--165, 1998.

\bibitem{Pel02}
A.~Pelc, ``Searching games with errors - fifty years of coping with liars,''
  {\em Theor. Comput. Sci.}, vol.~270, no.~1-2, pp.~71--109, 2002.

\bibitem{BDVY17}
N.~H. Bshouty, D.~Drachsler{-}Cohen, M.~T. Vechev, and E.~Yahav, ``Learning
  disjunctions of predicates,'' in {\em Proceedings of the 30th Conference on
  Learning Theory, {COLT} 2017, Amsterdam, The Netherlands, 7-10 July 2017},
  pp.~346--369, 2017.

\bibitem{BC16}
N.~H. Bshouty and A.~Costa, ``Exact learning of juntas from membership
  queries,'' in {\em Algorithmic Learning Theory - 27th International
  Conference, {ALT} 2016, Bari, Italy, October 19-21, 2016, Proceedings},
  pp.~115--129, 2016.

\bibitem{GRR08}
K.~D. Grave, J.~Ramon, and L.~D. Raedt, ``Active learning for high throughput
  screening,'' in {\em Discovery Science, 11th International Conference, {DS}
  2008, Budapest, Hungary, October 13-16, 2008. Proceedings}, pp.~185--196,
  2008.

\bibitem{DAG06}
A.~Z. Dudek, T.~Arodz, and J.~Galvez, ``Computational methods in developing
  quantitative structure-activity relationships (qsar): A review,'' {\em
  Combinatorial Chemistry \& High Throughput Screening}, vol.~9, no.~3,
  pp.~213--228, 2006.

\bibitem{GPP05}
G.~A. Landrum, J.~Penzotti, and S.~Putta, ``Machine-learning models for
  combinatorial catalyst discovery,'' {\em MRS Proceedings}, vol.~804,
  p.~JJ11.5, 2003.

\bibitem{BGS18}
R.~Ben{-}Basat, M.~Goldstein, and I.~Segall, ``Learning software constraints
  via installation attempts,'' {\em CoRR}, vol.~abs/1804.08902, 2018.

\bibitem{Bsh17}
N.~H. Bshouty, ``Exact learning from an honest teacher that answers membership
  queries,'' {\em Theor. Comput. Sci.}, vol.~733, pp.~4--43, 2018.

\bibitem{VF07}
V.~Feldman, ``Attribute-efficient and non-adaptive learning of parities and
  {DNF} expressions,'' {\em Journal of Machine Learning Research}, vol.~8,
  pp.~1431--1460, 2007.

\bibitem{KM93}
E.~Kushilevitz and Y.~Mansour, ``Learning decision trees using the fourier
  spectrum,'' {\em {SIAM} J. Comput.}, vol.~22, no.~6, pp.~1331--1348, 1993.

\bibitem{Jck97}
J.~C. Jackson, ``An efficient membership-query algorithm for learning {DNF}
  with respect to the uniform distribution,'' {\em J. Comput. Syst. Sci.},
  vol.~55, no.~3, pp.~414--440, 1997.

\bibitem{T17}
A.~Ta{-}Shma, ``Explicit, almost optimal, epsilon-balanced codes,'' in {\em
  Proceedings of the 49th Annual {ACM} {SIGACT} Symposium on Theory of
  Computing, {STOC} 2017, Montreal, QC, Canada, June 19-23, 2017},
  pp.~238--251, 2017.

\bibitem{BHL95}
A.~Blum, L.~Hellerstein, and N.~Littlestone, ``Learning in the presence of
  finitely or infinitely many irrelevant attributes,'' {\em J. Comput. Syst.
  Sci.}, vol.~50, no.~1, pp.~32--40, 1995.

\bibitem{BM95}
N.~H. Bshouty and Y.~Mansour, ``Simple learning algorithms for decision trees
  and multivariate polynomials,'' {\em {SIAM} J. Comput.}, vol.~31, no.~6,
  pp.~1909--1925, 2002.

\bibitem{B14}
N.~H. Bshouty, ``Testers and their applications,'' in {\em Innovations in
  Theoretical Computer Science, ITCS'14, Princeton, NJ, USA, January 12-14,
  2014}, pp.~327--352, 2014.

\bibitem{B15}
N.~H. Bshouty, ``Linear time constructions of some d -restriction problems,''
  in {\em Algorithms and Complexity - 9th International Conference, {CIAC}
  2015, Paris, France, May 20-22, 2015. Proceedings}, pp.~74--88, 2015.

\bibitem{NSS95}
M.~Naor, L.~J. Schulman, and A.~Srinivasan, ``Splitters and near-optimal
  derandomization,'' in {\em 36th Annual Symposium on Foundations of Computer
  Science, Milwaukee, Wisconsin, USA, 23-25 October 1995}, pp.~182--191, 1995.

\bibitem{ABI86}
N.~Alon, L.~Babai, and A.~Itai, ``A fast and simple randomized parallel
  algorithm for the maximal independent set problem,'' {\em J. Algorithms},
  vol.~7, no.~4, pp.~567--583, 1986.

\bibitem{NN90}
J.~Naor and M.~Naor, ``Small-bias probability spaces: Efficient constructions
  and applications,'' in {\em Proceedings of the 22nd Annual {ACM} Symposium on
  Theory of Computing, May 13-17, 1990, Baltimore, Maryland, {USA}},
  pp.~213--223, 1990.

\end{thebibliography}
\bibliographystyle{ieeetr}

\newpage
\appendix

\section{KM-algorithm}\label{KM}

Along the discussion, we will frequently need to estimate the expected value of random variables using Chernoff bound. Chernoff bound allows us to compute the sample size needed to ensure that the estimation is close to its mean with high probability. To make the analysis easy to follow for the reader, we bring Chernoff bound in the following Lemma.
\begin{lemma}(Chernoff).
Let $X_1, \cdots, X_m$ be independent random variables all with mean $\mu$ such that for all $i$, $X_i \in [-1,1]$. Then, for any $\lambda > 0$,
\begin{equation}
 Pr \left[ \left| \frac{1}{m}\mathop{\sum}_{i=1}^m X_i - \mu \right| \geq \lambda  \right] \leq 2e^{-\lambda^2m/2}.
\label{CHRNF}
\end{equation}
\end{lemma}

\subsection{The Discrete Fouriere Transform - DFT}
One of the powerful techniques used in the litrature for learning $\DT_d$ is the \emph{Discrete Fourier Transform - $DFT$}\cite{Bsh17}. 
To simplify the analysis, it is more convenient to change the range of the Boolean functions from $\{0,1\}$ to $\{-1, +1\}$, i.e. $g:\{0,1\}^n \rightarrow \{-1, +1\}$ by redefining the function $g$ from the standard representation of $f$ according to $g = 1-2f$. In this case, the zero-labeled leaves of $T(f)$ will be relabeled by $-1$ while the others remain untouched. Moreover, Definition~\ref{DTDefn} still holds for this case too with a minor change that $f\equiv -1$ is a decision tree instead of $f \equiv 0$. Similarily, the ``$+$'' sign still indicates arethmetic summation. Therefore, we can relate to $f$ as a real valued function $f:\{0,1\}^n \rightarrow \R$.
The following is an example of a decision tree of depth $3$

\begin{eqnarray}
f & = &  x_2(x_3\cdot (-1) + \bar{x_3}\cdot 1) + \bar{x_2} (x_1 \cdot 1 + \bar{x_1}(x_3\cdot 1 + \bar{x_3}\cdot (-1))) \nonumber \\
&=& x_2\bar{x_3} - x_2\bar{x_3}+ x_1\bar{x_2} + \bar{x_1}\bar{x_2}x_3 - \bar{x_1}\bar{x_2}\bar{x_3}. \label{fEx}
\end{eqnarray}

Denote by $\E[f]$ the expected value of $f(x)$ with respect to the uniform distribution on $x$.
Along the discussion, in case the expectation is taken over uniform distribution, we will omit the distibution notation.

 Let $\mathcal{F}$ denote the set of  real  functions on $\{0,1\}^n$. $\mathcal{F}$ is a $2^n$-dimensional real vector space with inner product defined as:
\begin{equation*}
(f,g) :=\frac{1}{2^n}\sum_{{x}\in \{0,1\}^n}f(x)g(x) = \E[f\cdot g].
\end{equation*}
The \emph{norm} of a function $f$ is defined by
$$||f|| := \sqrt{(f,f)}= \sqrt{\E[{f^2}]}.$$
For each $a\in\{0,1\}^n$ define the {\it parity function} $\chi_a(x)$:
\begin{equation*}
\chi_{{a}}(x) = (-1)^{a^Tx} = (-1)^{\sum_{i=1}^{n}a_ix_i}.
\label{eq:prty}
\end{equation*}
The functions $\{\chi_{a} \}_{{a}\in\{0,1\}^n}$ are an orthonormal basis for $\mathcal{F}$. Hence, for all $f\in \mathcal{F}$ we can write
 \begin{equation*}
f({x}) = \sum_{{a}\in \{0,1\}^n}\hat{f}(a)\cdot \chi_{{a}}({x})
\label{eq:FRSUM}
\end{equation*}
for $\hat{f}(a) \in \R$. The coefficients $\hat{f}(a)$ are called \emph{Fourier coeffcients} and,
\begin{equation*}
\hat{f}(a) = \E\left[f(x)\chi_{a}({x})\right].
\label{eq:FRCOEF}
\end{equation*}
By Parsaval's Theorem, we get that $\E[f^2] = \sum_{a\in \{0,1\}^n}\hat{f}^2(a)$. For Boolean functions defined over the range  $\{-1,+1\}$ we get that,
\begin{equation}
\E[f^2] = \sum_{{a}\in \{0,1\}^n}\hat{f}^2(a) = 1.
\label{eq:prsvl}
\end{equation}

We now show that:
\begin{lemma}\label{witnes1} Let $f=\sum_{a\in A}\lambda_a\chi_a(x)$ be any function where $\lambda_a\not=0$ for every $a\in A$. Then, the
relevant variables of $f$ can be found in time $O(|A|n)$. Moreover, for each relevant variable $x_i$, an assignment $b$ such that $f(w|_{x_i\gets 0})\not= f(w|_{x_i\gets 1})$, i.e., a witness for $x_i$, can be found in time $O(|A|n)$.
\end{lemma}
\begin{proof} First, since $\chi_a$ are orthonormal basis, $f=0$ if and only if $A=\emptyset$. Second, to find an assignment $w$ such that $f(w)\not=0$, we find $\xi_1\in\{0,1\}$ such that $f(x|_{x_1\gets \xi_1})\not=0$ and then recursively do that for $x_2,\ldots,x_n$. This takes $O(|A|n)$ steps.

Since $f=f_1+(-1)^{x_i}f_2$ where $f_1=\sum_{a\in A,a_i=0}\lambda_a \chi_a(x)$ and $f_2=\sum_{a\in A,a_i=1}$ $\lambda_a \chi_{a}(x|_{x_i\gets 0})$, we have $f(x|_{x_i\gets 0}) = f(x|_{x_i\gets 1})$ if and only if $f_2=0$. That is, $\{a\in A|a_i=1\}=\emptyset$. Therefore, $x_i$ is relevant in $f$ if and only if $a_i=1$ for some $a\in A$. Thus, the set of relevant variables of $f$ is $X=\cup_{a\in A}\cup_{a_i=1}\{x_i\}$.

To find a witness for $x_i$ we need to find an assignment $b$ such that $f(b|_{x_i\gets 1})\not= f(b|_{x_i\gets 1})$. That is $f(b|_{x_i\gets 1})- f(b|_{x_i\gets 1})\not=0$. Consider the function $g(x)=f(x|_{x_i\gets 1})- f(x|_{x_i\gets 1})$. Then we need to find an assignment $b$ such that $g(b)\not=0$. This can be done as above in $O(|A|n)$ steps.
\end{proof}


\subsection{From Decision Tree to DFT Representation}
There is a standard way to transform a Boolean function represented as a decision tree to DFT fomulation. Definition \ref{DTDefn} implies that every decision tree of size $s$ and depth $d$ can be written as a sum (in $\R$) of $s$ terms each of size at most $d$. Formally,
\begin{equation*}
f(x_1,\cdots, x_n) = \sum _{i=1}^{s}\alpha_iT_i,
\label{TERMDT}
\end{equation*}
such that for all $ 1 \leq i \leq s$,
\begin{equation*}
T_i = \ell_{i_1}\ell_{i_2}\cdots \ell_{i_{d'}},
\end{equation*}
for  $d' \leq d$, and $\ell_{i_j} = x_{i_j}$ or $\ell_{i_j} = \overline{x}_{i_j}$. Moreover, the parameter $\alpha_i$ is the value of the leaf in the path defined by the term $T_i$ in $T(f)$.

To formulate $f$ in the DFT representation, it is enough to rewrite each term $T_i$ and then sum the whole transformed terms into one formula. Over the real numbers $\R$, any positive literal $x_{i_j}$ in $T_i$ can be  expressed as $(1-(-1)^{x_{i_j}})/2$. In a similar fashion, any negative literal $\overline{x}_{i_j}$ translates to $(1+(-1)^{x_{i_j}})/2$. We apply this translation to all the terms $T_i$. The result is a DFT representation of $f$.
We demonstrate the process with an example. Let $f\in \DT_3$ be the decision tree described in (\ref{fEx}).  The function $f$ can be written as,
$$f(x_1,  x_2, x_3)= x_2x_3+ \overline{x}_2x_1+\overline{x}_2\overline{x}_1x_3-x_2\overline{x}_3 - \overline{x}_1\overline{x}_2\overline{x}_3$$
where the first three terms comply with the positive pathes in the decision tree and the other two comply with the negative ones.  We start with $x_2x_3$:
$$
\begin{array}{ll}
x_2x_3= & \left (\frac{1-(-1)^{x_2}}{2}\right ) \left(\frac{1-(-1)^{x_3}}{2}\right)\\
& \\
 &= \frac{1}{4} - \frac{(-1)^{x_2}}{4} -\frac{(-1)^{x_3}}{4}+ \frac{(-1)^{x_2+x_3}}{4}\\
& \\
 &= \frac{1}{4}\chi_{0000}-\frac{1}{4}\chi_{0100}-\frac{1}{4}\chi_{0010}+\frac{1}{4}\chi_{0110}.
\end{array}
$$
Similarily, we can rewrite $\overline{x}_2x_1$,
$$
\begin{array}{ll}
\overline{x}_2x_1= & \left (\frac{1+(-1)^{x_2}}{2}\right ) \left(\frac{1-(-1)^{x_1}}{2}\right)\\
& \\
&= \frac{1}{4} - \frac{(-1)^{x_1}}{4} +\frac{(-1)^{x_2}}{4}- \frac{(-1)^{x_2+x_1}}{4}\\
&\\
 &= \frac{1}{4}\chi_{0000}-\frac{1}{4}\chi_{1000}+\frac{1}{4}\chi_{0100}-\frac{1}{4}\chi_{1100}.
\end{array}
$$
To get the full DFT expression of $f$, we continue in the same fashion  for the other three terms. The construction described above imply the following results:
\begin{lemma}\label{Term}
Let $T$ be a term of size $d'$ then,
\begin{enumerate}
\item $T$ has a DFT representation that contains $2^{d'}$ non-zero coefficients $\hat{f}(a)$.
\item Each coefficient $\hat{f}(a)$ has the value $\pm\frac{1}{2^{d'}}$.
\item Each nonzero coefficient $\hat{f}(a)$ of $T$ satisfies $wt(a) \leq d'$.
\end{enumerate}
\label{lm:frco1}
\end{lemma}
\begin{proof} The proof follows directly from the construction of DFT from $T$.
\end{proof}

\begin{lemma} Let $f\in DT_{d}$, then for all ${a}\in \{0,1\}^n$,
\begin{equation*}
\hat{f}(a)=\frac{k}{2^d}
\end{equation*}
 for some  $ -2^d\leq k\leq 2^d.$
\label{lm:frco2}
\end{lemma}

\begin{proof} First notice that by Lemma~\ref{Term}, the Fourier coefficients of each term of size at most $d$ is of the form $k/2^{d}$. Therefore, the Fourier coefficients of the sum of terms of size at most $d$ is of the same form. The fact that $ -2^d\leq k\leq 2^d$ follows from the fact that $|\hat{f}(a)|\le 1$ which follows from Parsaval (\ref{eq:prsvl}).
\end{proof}
Using the fact that any $f\in \DT_d$ has at most $2^d$ terms each of size at most $d$ we can conclude:
\begin{corollary}\label{hhh}
Let $f \in \DT_d$, then $f$ has a DFT representation that contains at most $2^{2d}$ non-zero coefficients $\hat{f}(a)$ each having a value in $\{\pm k/2^d | k \in [2^d]\}$ and Hamming weight $wt(a) \leq d$.
\label{col:frco3}
\end{corollary}

\label{APX_A}
\subsection{The KM-algorithm}
\label{APX_B}
Kushilevitz and Mansour \cite{KM93} and Jackson ~\cite{Jck97} gave an adaptive algorithm that finds the non-zero coefficients in $poly(2^d, n)$ time and membership queries. Kushilevitz and Mansour algorithm (KM-algorithm) is a recursive algorithm that is given as an input an access to a membership query oracle $MQ_f$, confidence $\delta$ and a string $\alpha \in \{0,1\}^r$. Initially, KM-algorithm runs on the empty string $\alpha = \varepsilon$. The fundamental idea of the algorithm is based on the fact that for any $\alpha \in \{0,1\}^r$ it can be shown that \cite{Jck97},
\begin{equation}
F_{\alpha} \triangleq \mathop{\sum}_{x\in \{0,1\}^{n-r}} \hat{f}(\alpha x)^2 =\mathop{\mathbf{\E}}_{y,z \in \{0,1\}^r, x \in \{0,1\}^{n-r}}[f(yx)f(zx)\chi_{\alpha}(y)\chi_{\alpha}(z)],
\label{JKSN}
\end{equation}
where for $y=(y_1,\cdots, y_s)$ and $x=(x_1, \cdots, x_{t})$,  $yx =(y_1,\cdots, y_s,x_1, \cdots, x_{t})$. For Boolean functions $f:\{0,1\}^n \rightarrow \{-1,+1\}$, (\ref{eq:prsvl}) and (\ref{JKSN}) imply that $F_\alpha \leq 1$ for all $\alpha$ where  equality holds for $\alpha = \varepsilon$. Furthermore, for $f \in \DT_d$ we  prove,
\begin{lemma} Let $f\in \DT_d$. Then, for all $\alpha \in \{0,1\}^r$, $r \leq n$ there is $\ell \in [2^{2d}]\cup \{0\}$ such that,
$$ F_{\alpha} = \frac{\ell}{2^{2d}}$$
for $\ell \in [2^{2d}]\cup \{0\}.$
\label{LMAFALPHA}
\end{lemma}
\begin{proof}
Follows from Corollary~\ref{hhh} and Parsaval (\ref{eq:prsvl}).
\end{proof}
KM-algorithm uses divide and conquer technique with the identity in (\ref{JKSN}) to find the non-zero coefficients. Firstly, the algorithm computes $F_0$ and $F_1$ and initiates $T_1 = \{\xi \in \{0,1\} | F_\xi \neq 0\}$. At some stage, the algorithm holds a set $T_r = \{\alpha \in \{0,1\}^r | F_\alpha \neq 0\}$ . For each $\alpha\in T_r$, the algorithm computes $F_{\alpha  0}$ and $F_{\alpha 1}$. Since $F_{\alpha} = F_{\alpha 0} + F_{\alpha 1} \neq 0$, at least one of the terms is not zero. Therefore, the algorithm defines $T_{r+1}=\{\alpha\xi | \alpha \in T_r, \xi \in \{0,1\} ,  F_{\alpha \xi}\neq 0\} \neq \O$.  Using (\ref{JKSN}), when  $r=n$ we can conclude that  $F_{\alpha} = \hat{f}(\alpha)^2.$ That is, on the algorithm's termination,  $T_n = \{\alpha \in \{0,1\}^n | \hat{f}(\alpha)^2 \neq 0\}$ includes all the non-zero Fourier coefficients of $f$.
Observe that for $f\in \DT_d$,  Corollary \ref{col:frco3} implies that $f$ has at most $2^{2d}$ non-zero coefficients. Hence, the number of elements in $T_r$ is at most $2^{2d}$ for all $r$. The following lemma summarizes this observation.

\begin{lemma}\cite{KM93}.
Let $f\in\DT_d$. Then, for any $1\leq r \leq n$, $F_{\alpha} \neq 0$ for at most $2^{2d}$ strings $\alpha \in \{0,1\}^r$.
\label{falphaNum}
\end{lemma}
The immediate result of Lemma \ref{falphaNum} is that  KM-algorithm computes $F_{\alpha}$ at most $n \cdot 2^{2d}$ times. Furthermore, from Chernoff bound it follows that:
\begin{lemma}
Let $f\in\DT_d$. Let $\alpha \in \{0,1\}^r$ and $1 \leq r \leq n$. There is an algorithm $\cal{A}$ that with probability at least $1-\delta$ exactly computes the value of $F_{\alpha}$. Moreover, the algorithm asks $O\left(2^{4d}\cdot \log(\frac{1}{\delta})\right)$ membership queries and runs in time $O\left(2^{4d}\cdot n \cdot \log(\frac{1}{\delta})\right)$.
\label{falphaqueries}
\end{lemma}
\begin{proof}
Let $\alpha \in \{0,1\}^{r}$. Chernoff bound can be used to calculate the expectation from (\ref{JKSN}). Let $W_1, \cdots, W_m$ be $m$ random variables such that $W_i = f(y^{(i)}x^{(i)})f(z^{(i)}x^{(i)})\chi_{\alpha}(y^{(i)})\chi_{\alpha}(z^{(i)})\in\{-1,1\}$ for random uniform $y^{(i)}, z^{(i)} \in \{0,1\}^r$ and $x^{(i)} \in \{0,1\}^{n-r}$. Using Chernoff bound with  $\lambda =\frac{1}{2^{2d+1}}$ in (\ref{CHRNF}) gives,
$$\Pr\left[ \left |\frac{1}{m}\sum_{i=1}^{m}W_i - F_{\alpha}\right |\geq \frac{1}{2^{2d+1}}\right ] \leq 2e^{-m/{2^{4d+1}}} < \delta.$$
By choosing $m = O(2^{4d}\cdot \log({1}/{\delta}))$ it follows that with probability at least $1-\delta$,
$$\left |\frac{1}{m}\sum_{i=1}^{m}W_i - F_{\alpha}\right |< \frac{1}{2^{2d+1}}. $$
Furthermore, Lemma \ref{LMAFALPHA} states that the possible values of $F_{\alpha}$ vary in steps of $1/2^{2d}$. Since the error is smaller than $\frac{1}{2}\cdot \frac{1}{2^{2d}}$,  the exact value of $F_{\alpha}$ can be found by rounding to the nearest number of the form $\ell/2^{2d}$ where $\ell \in [2^{2d}]\cup \{0\}$.

Each membership query requires generating $O(n)$ random bits, therefore the time complexity is $O(2^{4d}\cdot n \cdot \log({1}/{\delta}))$.
\end{proof}
From Lemma \ref{falphaqueries}, it can be concluded that,
\begin{theorem}\cite{KM93}.
There is an adaptive Monte Carlo learning algorithm $\cal{B}$, that with probability at least $1-\delta$ exactly learns $\DT_d$. Moreover, the algorithm runs in time $O(2^{6d}\cdot n^2 \cdot(\log ({n}/{\delta}) +  d))$ and asks  $O(2^{6d}\cdot n\cdot (\log({n}/{\delta}) + d))$ membership queries.
\end{theorem}
\begin{proof} As mentioned before, the algorithm computes $F_{\alpha}$ at most $n\cdot 2^{2d}$ times. By the union bound, to get a failure probability of at most $\delta$, we need to evaluate each $F_{\alpha}$ with a failure probability ${\delta}/({n\cdot 2^{2d}})$ using  algorithm $\cal{A}$ from Lemma \ref{falphaqueries}. This implies that each evaluation of $F_{\alpha}$ requires $O(2^{4d} \cdot (\log({n}/{\delta}) + d))$ membership queries and runs in time $O(2^{4d} \cdot n \cdot (\log({n}/{\delta}) + d))$ . Since algorithm $\cal{B}$ invokes  $\cal{A}$ at most $n \cdot 2^{2d}$ times, the theorem follows.

\end{proof}

\subsection{Derandomization}
In this subsection, we describe KM-algorithm derandomization in details. We start by giving some preliminaries needed for the analysis. 
\subsubsection {Preliminaries}
Kushilevitz and Mansour use the method of $\lambda$\emph{-bias distributions} to derandomize their algorithm. The following definitions are useful to review this derandomization method. 

For any set $R$ we will use the notation $\mathbf{\E}_{R}[\cdot]$ to denote $\mathbf{\E}_{x \in R}[\cdot]$ with the uniform distribution over the elements of $R$.
Let $S = \{0,1\}^n$ and let $S'\subseteq S$,
\begin{defn}We say that $S'$ is a \emph{$d$-wise independent set} if for all uniformly random $x\in S'$ and all $1 \leq i_1 < \cdots <i_d\leq n$ and all $\xi_1, \cdots, \xi_d$ we have
$$\Pr\Big[ x_{i_1}=\xi_1, \cdots, x_{i_d}=\xi_d\Big] = \frac{1}{2^d}.$$
\end{defn}
Alon et al.\cite{ABI86} showed:
\begin{lemma}\label{ABI86} Let $w=\lceil \log n\rceil$, $2s+1\ge d$ and $m=ws+1$. Then, there is an $m\times n$ matrix $M$ over the field $F_2$ that can be constructed in $poly(n)$ time such that $S'=\{xM|x\in \{0,1\}^{m}\}$ is $d$-wise independent set.
\end{lemma}

\begin{defn} We say that $S'$ is \emph{$\lambda$-biased with respect to linear tests} if for all $\alpha\in \{0,1\}^n\backslash \{0^n\}$ it holds that,
$$\left | \E_{S'}[\chi_\alpha(x) \Big ] \right| \leq \lambda.$$
\end{defn}

Moreover, Ta-Shma\cite{T17} proves that:
\begin{lemma}\label{T17} There is a $\lambda$-biased with respect to linear tests $S'$ of size
$$\frac{n}{\lambda^{2+o(1)}}$$ that can be constructed in polynomial time $poly(n,1/\lambda)$.
\end{lemma}
\begin{defn} We say that $S'$ is \emph{$\lambda$-biased with respect to linear tests of size at most $d$} if for all $\alpha\in \{0,1\}^n\backslash \{0^n\}$ such that $wt(\alpha)\le d$ it holds that,
$$\left | \E_{S'}[\chi_\alpha(x) \Big ] \right| \leq \lambda.$$
\end{defn}

Naor and Nir\cite{NN90} prove that:
\begin{lemma}\label{NN90} Given a matrix $M$ as in Lemma~\ref{ABI86} of size $m\times n$ and a $\lambda$-biased with respect to linear tests set $\hat S\subseteq \{0,1\}^m$, one can, in polynomial time, construct a $\lambda$-biased set with respect to linear tests of size at most $d$ of size $|\hat S|$.
\end{lemma}
Taking $s=\lceil(d-1)/2\rceil$ in Lemma~\ref{ABI86}, we get $m=O(d\log n)$ and by Lemma~\ref{T17} and \ref{NN90} we conclude that:

\begin{lemma}\label{Combi} There is a $\lambda$-biased with respect to linear tests of size at most $d$, $S'$ of size
$$O\left(\frac{d\log n}{\lambda^{2+o(1)}}\right)$$ that can be constructed in polynomial time $poly(n,1/\lambda)$.
\end{lemma}


\subsubsection{KM-algorithm derandomization via $\lambda$-bias sample spaces}
The appropriateness of algorithm derandomization via $\lambda$-bias spaces requires showing that the output of the algorithm when its coin tosses are chosen from a uniform distribution, and the output of the algorithm when its coin tosses are chosen from $\lambda$-bias distribution are very similar. If this holds, a deterministic version of KM algorithm will replace the Chernoff-based computation of $F_{\alpha}$ by taking expectation over all elements of the $\lambda$-bias sample space. Consequently, a polynomially-small bias sample space is a mandatory requirement.
To show ``similarity" between the randomized algorithm and the deterministic one, we start with the following lemmas. By Parsaval's Theorem we have:
\begin{lemma}
Let $f:S \rightarrow \mathbb{R}$ and $S = \{0,1\}^n$. Then, $$\mathop{\mathbf{\E}_{x\in S}[f(x)^2] = \sum_{a\in S}\hat{f}(a)^2}.$$
\label{LM1}
\end{lemma}
\begin{defn} Let $f$ be any Boolean function. $L_1$-norm of $f$ is defined as,
$$L_1(f) \triangleq \sum_z|\hat{f}(z)|.$$
\end{defn}
By subadditivity of the absolute value we have $L_1(f+g)\le L_1(f)+L_1(g)$ and by Lemma~\ref{Term}, for any term $T$, $L_1(T)=1$. Since each $f\in\DT_d$ is a sum of at most $2^d$ terms we have:
\begin{lemma}
Let $f\in\DT_d$. Then, $L_1(f) \leq 2^{d}$.
\label{L1NORM}
\end{lemma}

\begin{lemma}
Let $f:S \rightarrow \mathbb{R}$ for $S = \{0,1\}^n$, and $S' \subseteq S$. Let
\begin{eqnarray}\label{es01} \lambda =\mathop{\mathop{\mathop{max}_{a,b\in S}}_{a\neq b}}_{\hat{f}(a),\hat{f}(b)\neq 0}|\mathbf{\E}_{S'}[\chi_{a+b}(x)]|,
\end{eqnarray}
then,
$$|\mathop{\mathbf{\E}}_{S'} [f(x)^2] - \mathop{\mathbf{\E}}_{ S}[f(x)^2]| \leq (L_1(f)^2 - \mathop{\mathbf{\E}}_{S'}[f(x)^2])\lambda \leq L_1(f)^2\lambda.$$
\label{LM2}
\end{lemma}
\begin{proof}
Using Lemma \ref{LM1} we have,
\begin{eqnarray*}
\mathop{\mathbf{\E}}_{S'}[f(x)^2]& =& \mathop{\mathbf{\E}}_{S'}\Big[\Big(\mathop{\sum}_{a\in S}\hat{f}(a)\chi_a(x)\Big )^2\Big] \\
& =&  \mathop{\mathbf{\E}}_{S'}\Big[\mathop{\sum}_{a,b\in S}\hat{f}(a)\hat{f}(b)\chi_{a+b}(x)\Big] \\
& =& \mathop{\sum}_{a \in S}\hat{f}(a)^2 + \mathop{\mathop{\sum}_{a,b\in S}}_{a\neq b}\hat{f}(a)\hat{f}(b)\mathbf{\E}_{S'}[\chi_{a+b}(x)]\\
& = & \mathop{\mathbf{\E}}_{S}[f(x)^2] +  \mathop{\mathop{\sum}_{a,b\in S}}_{a\neq b}\hat{f}(a)\hat{f}(b)\mathbf{\E}_{S'}[\chi_{a+b}(x)].
\end{eqnarray*}
Thus we can say,
\begin{eqnarray*}
|\mathop{\mathbf{\E}}_{S'}[f^2] - \mathop{\mathbf{\E}}_{S}[f^2]|& =& \left| \mathop{\mathop{\sum}_{a,b\in S}}_{a\neq b}\hat{f}(a)\hat{f}(b)\mathbf{\E}_{S'}[\chi_{a+b}(x)]\right| \\
& \leq &  \mathop{\mathop{\sum}_{a,b\in S}}_{a\neq b}|\hat{f}(a)|\cdot |\hat{f}(b)| \cdot |\mathbf{\E}_{S'}[\chi_{a+b}(x)]| \\
&\leq & \mathop{\mathop{\mathop{max}_{a,b\in S}}_{a\neq b}}_{\hat{f}(a),\hat{f}(b)\neq 0}|\mathbf{\E}_{S'}[\chi_{a+b}(x)]| \mathop{\mathop{\sum}_{a,b\in S}}_{a\neq b}|\hat{f}(a)|\cdot |\hat{f}(b)|\\
& \leq & \lambda  \cdot \left (\Big(\mathop{\sum}_{a\in S}|\hat{f}(a)|\Big)^2-\mathop{\sum}_{a\in S}|\hat{f}(a)|^2 \right) \\
&  \leq & \lambda \cdot \Big (L_1(f)^2 - \mathop{\mathbf{\E}}_{S}[f^2] \Big ) \leq  \lambda \cdot L_1(f)^2.
\end{eqnarray*}
\end{proof}

\begin{defn} Let $S_1 = \{0,1\}^{k}$ and $S_2 = \{0,1\}^{n-k}$  for $ 0 \leq k \leq n$. Let $S_1'\subseteq S_1$,  $\alpha, y \in S_1$ and $x\in S_2$. Define $h_{S_1'}: S_2 \rightarrow \mathbb{R}$ such that,

\begin{equation}
h_{S_1'}(x) \triangleq \mathop{\mathbf{\E}}_{y\in S_1'}[f(yx)\chi_{\alpha}(y)].
\label{dfn6}
\end{equation}
\label{d1}
\end{defn}
When choosing $S_1'=S_1$,  we will use the notation  $h = h_{S_1}$. In this case $h$ is defined by taking the expectation on uniform distribution over $S_1$.
\begin{lemma}
Let $S=\{0,1\}^n$ and $S_1$, $S_2$ as in Definition \ref{d1}. Let $S_1' \subseteq S_1$.  Let $h_{S_1'}$ and $h = h_{S_1}$  be as defined in (\ref{dfn6}). Then,
\begin{enumerate}
\item $L_1(h) \leq L_1(f).$
\item $$h_{S_1'} - h = \mathop{\sum}_{a_2\in S_2}\Big (  \mathop{\mathop{\sum}_{a_1\in S_1}}_{a_1\neq \alpha}\hat{f}(a_1a_2)\mathop{\mathbf{\E}}_{y\in S_1}[\chi_{a_1+\alpha}(y)]  \Big )\chi_{a_2}(x).$$
\end{enumerate}
\label{lma5}
\end{lemma}
\begin{proof} We first prove {\it 1.} Using Definition \ref{d1} and the Fourier transform of $f$ and the fact that $\mathbf{\E}_{S_1}[\chi_{a}(y)]=0$ for all $a\neq 0$ when the expectation is taken over uniform distribution we get,

\begin{eqnarray*}
h(x)&=& \mathop{\mathbf{\E}}_{y\in S_1}\left[\mathop{\sum}_{a\in S}\hat{f}(a)\chi_{a}(yx)\chi_{\alpha}(y)\right]\\
&=&\mathop{\sum}_{a_2\in S_2}\Big (\mathop{\sum}_{a_1\in S_1}\hat{f}(a_1a_2)\mathop{\mathbf{\E}}_{y\in S_1}[\chi_{a_1+\alpha}(y)]  \Big ) \chi_{a_2}(x)\\
&=&\mathop{\sum}_{a_2\in S_2} \hat{f}(\alpha a_2)\chi_{a_2}(x).\\
\end{eqnarray*}
The last equation follows because $\E[\chi_{a_1+\alpha}(y)]=0$ for $a_1\not=\alpha$.
Hence, we get that $\hat{h}(a) = \hat{f}(\alpha a)$. Therefore,
$$L_1(h) = \mathop{\sum}_{a\in S_2}|\hat{h}(a)| = \mathop{\sum}_{a\in S_2}|\hat{f}(\alpha a)| \leq \mathop{\sum}_{b\in S}|\hat{f}( b)| = L_1(f).$$

We now prove {\it 2.} Similarily, we can say,
\begin{eqnarray*}
h_{S_1'}(x)&=& \mathop{\sum}_{a_2\in S_2}\Big (\mathop{\sum}_{a_1\in S_1}\hat{f}(a_1a_2)\mathop{\mathbf{\E}}_{y\in S_1'}[\chi_{a_1+\alpha}(y)]  \Big ) \chi_{a_2}(x)\\
&=&\mathop{\sum}_{a_2\in S_2} \hat{f}(\alpha a_2)\chi_{a_2}(x) + \mathop{\sum}_{a_2\in S_2}\Big (\mathop{\sum}_{\alpha \neq a_1\in S_1}\hat{f}(a_1a_2)\mathop{\mathbf{\E}}_{y\in S_1'}[\chi_{a_1+\alpha}(y)]  \Big ) \chi_{a_2}(x)\\
&=&h(x)+ \mathop{\sum}_{a_2\in S_2}\Big (\mathop{\sum}_{\alpha \neq a_1\in S_1}\hat{f}(a_1a_2)\mathop{\mathbf{\E}}_{y\in S_1'}[\chi_{a_1+\alpha}(y)]  \Big ) \chi_{a_2}(x).
\end{eqnarray*}
\end{proof}

\begin{lemma} Let $S_1'\subseteq S_1$, and $S_2'\subseteq S_2$ and $h_{S_1'}, h_{S_1}=h$ be as in Definition \ref{d1}. Let $\lambda_1 =  \mathop{max}_{c \in S_1, c\neq 0}|\mathop{\mathbf{\E}}_{S_1'}[\chi_{c}]|$ and $\lambda_2 =  \mathop{max}_{c \in S_2, c\neq 0}|\mathop{\mathbf{\E}}_{S_2'}[\chi_{c}]|$ then,
$$ |\mathop{\mathbf{\E}}_{S_2'}[h^2_{S_1'}] - \mathop{\mathbf{\E}}_{S_2}[h^2]| \leq  2L_1(f)\lambda_1+  L_1(f)^2\lambda_2.$$
\label{ERRBOUND}
\end{lemma}

\begin{proof}
By the subadditivity of the absolute value function,
\begin{equation}
|\mathop{\mathbf{\E}}_{S_2'}[h^2_{S_1'}] - \mathop{\mathbf{\E}}_{S_2}[h^2]| \leq |\mathop{\mathbf{\E}}_{S_2'}[h^2_{S_1'}] - \mathop{\mathbf{\E}}_{S_2'}[h^2]| + |\mathop{\mathbf{\E}}_{S_2'}[h^2] - \mathop{\mathbf{\E}}_{S_2}[h^2]|.
\label{eq1}
\end{equation}
Since $-1 \leq h_{S_1'} \leq 1$ and $-1 \leq h \leq 1$ we can see that,
\begin{eqnarray}
 |\mathop{\mathbf{\E}}_{S_2'}[h^2_{S_1'}] - \mathop{\mathbf{\E}}_{S_2'}[h^2]|& = &  |\mathop{\mathbf{\E}}_{S_2'}[(h_{S_1'}-h)(h_{S_1'}+h)]|  \nonumber \\
& \leq &  \mathop{\mathbf{\E}}_{S_2'}[|h_{S_1'}-h|\cdot |h_{S_1'}+h|] \nonumber \\
& \leq & 2 \mathop{\mathbf{\E}}_{S_2'}[|h_{S_1'}-h|].
\label{eq11}
\end{eqnarray}
By Lemma \ref{lma5},
\begin{eqnarray}
\mathop{\mathbf{\E}}_{x\in S_2'}[|h_{S_1'}-h|] & =& \mathop{\mathbf{\E}}_{x \in S_2'}\left[\left|\mathop{\sum}_{a_2\in S_2}\Big (\mathop{\sum}_{\alpha \neq a_1\in S_1}\hat{f}(a_1a_2)\mathop{\mathbf{\E}}_{y\in S_1'}[\chi_{a_1+\alpha}(y)]  \Big ) \chi_{a_2}(x)\right|\right] \nonumber \\
& \leq &  \mathop{\mathbf{\E}}_{x\in S_2'}\left[\mathop{\sum}_{a_2\in S_2}\Big (\mathop{\sum}_{\alpha \neq a_1\in S_1}|\hat{f}(a_1a_2)|\cdot |\mathop{\mathbf{\E}}_{y\in S_1'}[\chi_{a_1+\alpha}(y)]|  \Big )\cdot 1\right] \nonumber \\
& = & \mathop{\sum}_{a_2\in S_2}\Big (\mathop{\sum}_{\alpha \neq a_1\in S_1}|\hat{f}(a_1a_2)|\cdot |\mathop{\mathbf{\E}}_{y\in S_1'}[\chi_{a_1+\alpha}(y)]|  \Big ) \label{eq2} \\
& \leq & L_1(f) \cdot \lambda_1. \nonumber
\end{eqnarray}
Hence, using (\ref{eq11}) and (\ref{eq2}) we conclude,
\begin{equation}
|\mathop{\mathbf{\E}}_{S_2'}[h^2_{S_1'}] - \mathop{\mathbf{\E}}_{S_2'}[h^2]| \leq 2\cdot L_1(f)\cdot \lambda_1.
\label{eq3}
\end{equation}
On the other hand, by lemmas \ref{LM2} and \ref{lma5} we can say that,
\begin{equation}
 |\mathop{\mathbf{\E}}_{S_2'}[h^2] - \mathop{\mathbf{\E}}_{S_2}[h^2]|  \leq  L_1(h)^2 \cdot\lambda_2 \leq L_1(f)^2\cdot \lambda_2.
\label{eq4}
\end{equation}
By combining  (\ref{eq1}), (\ref{eq3}) and (\ref{eq4}) the Lemma follows.
\end{proof}
Using Definition \ref{d1} we can rewrite Equation \ref{JKSN}:
\begin{equation}
F_{\alpha}=\mathop{\mathbf{\E}}_{x \in S_2}[(\mathop{\mathbf{\E}}_{y \in S_1}[f(yx)\chi_{\alpha}(y)])^2] = \mathop{\mathbf{\E}}_{x \in S_2}[h_{S_1}^2] =  \mathop{\mathbf{\E}}_{x \in S_2}[h^2].
\label{FJKSN_2}
\end{equation}
Derandomization via bias spaces technique is done by choosing polynomially small $S_1'\subseteq S_1$ and $S_2'\subseteq S_2$ such that the expectations needed to evaluate $F_\alpha$ can be done by taking the expectation over the small domains $S_1'$ and $S_2'$. Let $F_{\alpha}'$ denote the estimatioin of $F_\alpha$  derived from (\ref{FJKSN_2}) by taking the expectations over $S_1'$ and $S_2'$. That is,
\begin{equation*}
F_\alpha ' = \mathop{\mathbf{\E}}_{x \in S_2'}[(\mathop{\mathbf{\E}}_{y \in S_1'}[f(yx)\chi_{\alpha}(y)])^2] = \mathop{\mathbf{\E}}_{x \in S_2'}[h_{S_1'}^2].
\label{FALPHATG}
\end{equation*}
Let $\Delta(F_{\alpha})$ denote the error in evaluating $F_\alpha$. Namely,
\begin{equation*}
\Delta(F_{\alpha}) \triangleq|F_\alpha - F_\alpha '|.
\end{equation*}

To get good approximation for $F_{\alpha}$, we  require the value $\Delta(F_\alpha)$ to be small with high probability.
Moreover, Lemma \ref{LMAFALPHA} implies that  requiring $\Delta(F_\alpha) = \frac{1}{2^{2d+1}}$ and rounding the value of $F_{\alpha}'$ to the closest $\ell/2^{2d}$ gives accurate evaluation of $F_\alpha$. By lemmas \ref{L1NORM} and \ref{ERRBOUND} we choose $S_1'$ and $S_2'$ such that,
$$\Delta(F_{\alpha})  = |\mathop{\mathbf{\E}}_{S_2'}[h^2_{S_1'}] - \mathop{\mathbf{\E}}_{S_2}[h^2]| \le  2L_1(f)\lambda_1+  L_1(f)^2\lambda_2 \le 2^{d+1}\lambda_1 + 2^{2d}\lambda_2 \le \frac{1}{2^{2d+1}}, $$
for $\lambda_1 =  \mathop{max}_{c \in S_1, c\neq 0}|\mathop{\mathbf{\E}}_{S_1'}[\chi_{c}(x)]|$ and $\lambda_2 =  \mathop{max}_{c \in S_2, c\neq 0}|\mathop{\mathbf{\E}}_{S_2'}[\chi_{c}(x)]|$.
 The number of membership queries needed in these settings is $|S_1'|\cdot |S_2'|$ for each estimation of $F_\alpha$.  It is easy see that we get optimal sizes when $\lambda_1 = 1/2^{3d+3}$ and $\lambda_2=1/2^{4d+2}$.

Corollary \ref{col:frco3} implies that the non-zero Fourier coefficients of $f\in \DT_d$ have Hamming weight of at most $d$.  All the assignments $c$ in the expectations of $\chi_c$ are
of the form $c=a+b$ where $\hat f(a),\hat f(b)\not=0$ and therefore $wt(c)\le 2d$. See (\ref{es01}) and (\ref{eq2}). Therefore, for $\lambda_1$ we need a set that is $\lambda_1$-biased with respect to linear tests of size at most $2d$, and for $\lambda_2$, a set that is $\lambda_2$-biased with respect to linear tests of size at most $2d$. By Lemma~\ref{Combi}, such sets can be constructed in $poly(2^d,n)$ time, and are of sizes
$d{2^{6d+o(d)}\log n}{}$ and $d{2^{8d+o(d)}\log n}{}$ respectively. This gives $\tilde{O}(2^{14d+o(d)} )\log ^2 n$ queries to exactly find each $F_\alpha$.

As mentioned in the KM-algorithm we need to compute $F_\alpha$ for $n2^{2d}$ $\alpha$'s and therefore

\begin{theorem}
There is an adaptive deterministic learning algorithm that exact learns $\DT_d$ in $poly(2^{d},n)$ time and asks $\tilde O(2^{16d+o(d)})n \cdot \log^2n$ membership queries.
\end{theorem}
Using the reduction introduced in \cite{BHL95} for decision trees we have,
\begin{theorem}
There is an adaptive deterministic learning algorithm that exact learns $\DT_d$ in $poly(2^d,n)$ time and asks $\tilde O(2^{18d+o(d)})\cdot \log n$ membership queries.
\end{theorem}

\end{document}